\documentclass{article}
\usepackage{graphicx} % Required for inserting images
\usepackage{amsmath,amssymb, amsthm, amsfonts}%
\usepackage{mathtools}
\usepackage{lipsum}
\usepackage[hidelinks]{hyperref}
\usepackage{bbm}
\usepackage{bm}
\usepackage{multicol}
\usepackage{mwe}
\usepackage{appendix}
\usepackage{geometry}
\usepackage{biblatex}
\usepackage{enumitem}
\usepackage{comment}
\usepackage{relsize}
\usepackage{booktabs}       % professional-quality tables
\usepackage{multirow}
\usepackage{makecell}
\usepackage{authblk}

\theoremstyle{definition}
\newtheorem{definition}{Definition}[section]
\theoremstyle{definition}
\newtheorem{setting}[definition]{Setting}
\theoremstyle{definition}

\newtheorem{assumptions}[definition]{Assumptions}
\theoremstyle{plain}
\newtheorem{theorem}[definition]{Theorem}
\newtheorem{lemma}[definition]{Lemma} 
\newtheorem{proposition}[definition]{Proposition} 
\newtheorem{corollary}[definition]{Corollary}
\theoremstyle{definition}

\newtheorem{remark}[definition]{Remark}

\newcommand{\R}{\mathbb{R}}
\newcommand{\N}{\mathbb{N}}

\newcommand{\C}{\mathbb{C}}

\newcommand{\E}{\mathbb{E}}
\newcommand{\1}{\mathbbm{1}}

\newcommand{\bary}{\Bar{Y}}
\newcommand{\barz}{\Bar{E}}
\newcommand{\barw}{\Bar{W}}
\newcommand{\bare}{\Bar{E}}

\newcommand{\clip}{\operatorname{clip}}
\usepackage{xcolor}

\title{Is Variational Monte Carlo Robust?\\[0.2em]\smaller{}
Sharp Moment Thresholds and Heavy-tailed Stochastic Optimization}
 \author{Philipp Grohs$^{1,2}$ and Davide Nobile$^1$}
 
\date{%
    $^1$ Faculty of Mathematics, University of Vienna\\
    firstname.lastname@univie.ac.at\\
    $^2$ RICAM, Austrian Academy of Sciences\\
    firstname.lastname@oeaw.ac.at\\
    [2ex]%
    \today
}

\begin{document}

\maketitle
\begin{abstract}
    Variational Monte Carlo (VMC) is a central algorithm in electronic structure theory and has gained renewed importance through modern neural-network ans\"atze such as FermiNet. At its core, VMC seeks ground states by minimizing the Rayleigh quotient by stochastic optimization. In this work, we show that the resulting stochastic optimization problem is intrinsically governed by the nodal geometry of the underlying wave function. More precisely, we establish that properties of the nodal set determine the integrability of the local energy and gradient estimators that drive VMC. For broad and practically relevant ansatz classes, including Slater-Jastrow wave functions with variable-exponent Slater-type orbitals, we prove that these estimators are generically heavy-tailed and fail to admit higher moments. At the same time, for general analytic ans\"atze, we prove weak moment bounds for the relevant estimators and identify precise low-moment regimes, showing how generic and degenerate nodal structures lead to different integrability thresholds. Building on this analysis, we introduce a new robust variant of VMC -- coined PS-Clip-VMC -- which is based on clipping both the local energies and the per-sample gradients. We prove that PS-Clip-VMC converges both in expectation and with high probability in the weak moment regime of VMC. The robustness of our method is confirmed experimentally by training FermiNet on atoms with up to 18 electrons.
 \end{abstract}
\section{Introduction}
\subsection{Motivation}
A central computational problem in quantum chemistry concerns the approximation of ground states and ground state energies of many-particle Hamiltonians. If
$\mathcal{H}:\ H^2(\Omega)\to L^2(\Omega)$ is a self-adjoint Hamiltonian bounded from below, its \emph{ground-state energy} $E_0$
is characterized by the Rayleigh--Ritz variational principle
\begin{equation}\label{eq:RayleighRitz}
        E_0
        :=
        \inf_{\psi\in \mathfrak{d}(\mathcal{H})\setminus \{0\}}
        \frac{\langle \mathcal{H}\psi,\psi\rangle}{\|\psi\|_{L^2}^2},
\end{equation}
with $\mathfrak{d}(\mathcal{H})$ the \emph{form-domain} of $\mathcal{H}$; typically we have $\mathfrak{d}(\mathcal{H})=H^1(\Omega)$ \cite{teschl2014mathematical}.
If the infimum is attained, then every minimizer is an eigenfunction associated
with the lowest eigenvalue of $\mathcal{H}$ -- a so-called \emph{ground state}. 
In many-particle systems we usually have that $\Omega = \R^{d\times N}$ (with $d$ the particle dimension and $N$ the number of particles) and (in dimensionless units) $\mathcal{H} = -\Delta + \sum_{i,j=1}^N V(x_i,x_j) + \sum_{i=1}^N V_{\mbox{ext}}(x_i)$, where $V:\R^d\times \R^d\to \R$ describes particle interactions and $V_{\mbox{ext}}:\R^d\to \R$ describes an external potential. If the particles are indistinguishable and fermionic, there is an additional restriction, namely $\psi\in H^1_a(\R^{d\times N})$, where we denote by $H_a^1(\R^{d\times N})$ the space of all functions $\psi\in H^1(\R^{3\times N})$ with $$\psi(x_{\pi(1)},\dots , x_{\pi(N)})=\mathrm{sgn}(\pi)\psi(x_1,\dots , x_N)$$ for all permutations $\pi$ of $\{1,\dots , N\}$, all $(x_1,\dots , x_N)\in \R^{d\times N}$ and $\mathrm{sgn}(\pi)$ denoting the signature of the permutation $\pi$.
We generally refer to multi-particle systems restricted to the space of antisymmetric functions as \emph{fermionic multi-particle systems} and note that for such systems, the problem \eqref{eq:RayleighRitz} amounts to computing 
\begin{equation}\label{eq:RayleighRitzElectronic}
    E_0
        :=
        \inf_{\psi\in H^1_a(\R^{3\times N})\setminus \{0\}}
        \frac{\langle \mathcal{H}\psi,\psi\rangle}{\|\psi\|_{L^2}^2}.
\end{equation}

A particularly important example is given by the (Born-Oppenheimer) multi-electron Hamiltonian 
\begin{equation}\label{eq:BOHamiltonian}
    \mathcal{H}^{BO}:=-\Delta + \sum_{i,j=1}^{N}\frac{1-\delta_{ij}}{|x_i - x_j|} + \sum_{I,J=1}^M\frac{Z_IZ_J(1-\delta_{IJ})}{|R_I - R_J|} - \sum_{i=1}^N\sum_{I=1}^M\frac{Z_I}{|x_i - R_I|}
\end{equation}
which describes electronic properties of a molecule with $M$ nuclei $(R_1,\dots , R_M)\in \R^{3\times M}$ and charges $(Z_1,\dots , Z_M)\in \mathbb{N}^{M}$. The multi-electron Hamiltonian acts on the \emph{electronic wavefunction} $\psi \in H^2(\R^{3\times N})$ which depends on the $n$ electron coordinates $(x_1,\dots , x_N)\in \R^{3\times N}$ \cite{helgaker2013molecular,szabo2012modern}. Since electrons are indistinguishable fermionic particles, (and ignoring spin) the electronic wavefunction is antisymmetric with respect to permutations of the electron coordinates. Finding efficient algorithms to accurately compute the ground-state energies and ground states of \eqref{eq:BOHamiltonian} (along with excited states corresponding to higher eigenvalues and eigenvectors) is one of the key open problems in computational chemistry. Its solution would enable the simulation and prediction of all nonrelativistic properties of molecules from first principles without resorting to expensive and time-consuming experiments.

In the case of the Born-Oppenheimer multi-electron Hamiltonian \eqref{eq:BOHamiltonian}, the solution of \eqref{eq:RayleighRitzElectronic} poses formidable challenges: First, the dimension of the problem scales linearly in the number of electrons. This means that even for small molecules with only tens or hundreds of electrons, solving \eqref{eq:RayleighRitzElectronic} requires solving a \emph{PDE eigenvalue problem on a computational domain with hundreds or even thousands of dimensions}. 
Second, many chemical applications require exceedingly accurate approximations of $E_0$ to within a few $\mathrm{mHa}$ (millihartree) while the absolute energy $E_0$ may be of the order of several hundred Hartree. In other words, practically useful computations often need to \emph{approximate $E_0$ to within a relative accuracy of $6-7$ digits} \cite{gerard2024deepLearningVMCChapter}.  

In recent years, so-called variational monte-carlo (VMC) algorithms have emerged as a new state of the art in terms of accuracy, see for example \cite{pfau2020ab,hermann2020deep,renaud2025qmctorch,scherbela2023budget,scherbela2023foundationPreprint,scherbela2024transferableMolecules,scherbela2025fire,gerard2022goldStandard,gerard2024deepLearningVMCChapter,gerard2025transferableSolids,hermann2023abinitioReview,gao2022pesnet,gao2023generalizing,gao2023samplingFree}. These algorithms build on a parametric function class $\Psi=\{\psi_\theta:\ \theta\in \Theta\}\subset H^1_a(\R^{3\times N})\setminus \{0\}$ and $\Theta\subset \R^P$ (for example consisting of neural networks) and approximate the true ground state and ground-state energy by minimizing the \emph{loss}
\begin{equation}\label{eq:VMCLoss}
    L(\theta):=\frac{\langle \mathcal{H}\psi_\theta,\psi_\theta\rangle}{\|\psi_\theta\|_{L^2}^2}.
\end{equation}
Since evaluating $L(\theta)$ or $\nabla_\theta L(\theta)$ requires the calculation of high-dimensional integrals, stochastic representations and corresponding stochastic optimization algorithms have to be used: with 
$$X_\theta\sim \frac{|\psi_\theta(x_1,\dots , x_N)|^2}{\|\psi_\theta\|_{L^2}^2},\  E_\theta:=\frac{\mathcal{H}\psi_\theta(X_\theta)}{\psi_\theta(X_\theta)}, \ W_\theta:=\frac{\nabla_\theta \psi_\theta(X_\theta)}{\psi_\theta(X_\theta)}\ \mbox{and}\ Z_\theta:= E_\theta \cdot W_\theta - \mathbb{E}\left[E_\theta\right]\cdot \mathbb{E}\left[ W_\theta\right]$$ we have that 
$$
    L(\theta) = \mathbb{E}[E_\theta]\quad \mbox{and}\quad \nabla_\theta L(\theta) = 2\mathbb{E}[Z_\theta].
$$
VMC algorithms leverage these stochastic representations by first sampling from $X_\theta$ via MCMC and then estimating $\nabla L(\theta)$ from suitable finite sample estimators of $\mathbb{E}[Z_\theta]$ to compute parameter updates \cite{toulouse2016introduction,becca2017quantum}. The success of this approach hinges on the quality with which the true gradient can be approximated by its respective stochastic approximation. Indeed, if the random variables $E_\theta$ and $Z_\theta$ satisfy sufficient moment assumptions,  \cite{abrahamsen2024convergence} recently proved convergence of VMC to stationary points of $L$. 

\subsection{Contributions}
The main contributions of this paper follow three threads which we now summarize.
\subsubsection{Non-Existence of Moments}
Our first main result uncovers yet another challenge in VMC: the moment assumptions of \cite{abrahamsen2024convergence} are in general not satisfied and both $E_\theta$ and $Z_\theta$ are severely heavy-tailed. More precisely, in Theorem \ref{thm:mainmomentslater} we prove that even for the simple case of $\Psi$ consisting of any (sufficiently expressive) Slater-Jastrow Ansatz with Slater-Type orbitals \cite{renaud2025qmctorch} the set of parameters with heavy tails on $E_\theta$ and $Z_\theta$
is residual (in fact, we show that it contains a dense open subset). For $h_1,\dots , h_N\in C(\R^3)$ and $(x_1,\dots , x_N)\in \R^{3\times N}$ we write
$$
    |h_1,\dots , h_N\rangle (x_1,\dots, x_N):=\det\begin{pmatrix}h_1(x_1)& \dots & h_N(x_1)\\\vdots & \ddots & \vdots \\h_1(x_N) & \dots & h_N(x_N)\end{pmatrix}.
$$
Then, we have the following result.
\begin{theorem}[Colloquial Version of Theorem \ref{thm:mainmomentslater}]
    With $Y_{l,m}$ denoting the spherical harmonics and
    %%
    %$$
    %    h_{\alpha_i,l,m,k}(x):=Y_{l,m}\left(\frac{x}{|x|}\right)|x|^{l+k}e^{-\alpha_i|x|},\quad l=0,\dots,L,\ m=-l,\dots , l,\ k=0,\dots, K-1,\ i=1,\dots , I.
    %$$
    $\bm{\alpha}\in [0,\infty)^I$, $\mathbf{c} =\left(c_{l,m,k,i}\right)_{l=0,\dots,L,\ m=-l,\dots , l,\ k=0,\dots, K-1,\ i=1,\dots , I}\in \R^{(L+1)^2KI}$ denote
    $$
        h_{(\bm{\alpha},\mathbf{c})}(x):=\sum_{l,m,k,i}c_{l,m,k,i}Y_{l,m}\left(\frac{x}{|x|}\right)|x|^{l+k}e^{-\alpha_i|x|},
    $$
    the corresponding \emph{Slater-Type orbital} and let
    $\mathcal{J}_\beta $ be a symmetric and positive \emph{Jastrow factor}, parametrized by $\beta\in \R^J$.

    For $\theta:=\left(\beta,((\bm{\alpha}_a^b,\mathbf{c}_a^b))_{a=1,\dots , N,\ b=1,\dots , B}\right)\in \R^J\times [0,\infty)^{NBI}\times \R^{NB(L+1)^2KL}=:\Theta$ denote 
    $$
        \psi_\theta(\mathbf{x})\coloneqq\mathcal{J}_\beta(\mathbf{x})\cdot\left(\sum_{b=1}^B |h_{(\bm{\alpha}_1^b,\mathbf{c}_1^b)},\dots ,h_{(\bm{\alpha}_N^b,\mathbf{c}_N^b)}\rangle(\mathbf{x}) \right).
    $$
    Then, whenever $L\geq 1$ and $KI\geq 3N-3$, the set
    $$
        \left\{\theta\in \Theta:\ \mathbb{E}[|E_\theta|^3]=\infty \quad \land \quad \mathbb{E}[|Z_\theta|^{3/2}]=\infty\right\}
    $$
    is residual.
\end{theorem}
Our result in particular shows that the gradient random variable $Z_\theta$ generically does not possess finite variance, rendering standard stochastic approximations insufficient.  The main difficulty in the proof of Theorem \ref{thm:mainmomentslater} is to show that every neighbourhood of every $\theta$ in $\Theta$ contains a parameter $\mu$ such that $\psi_\mu$ has a regular zero $(x_1,\dots , x_n)\in \R^{3\times n}$ with $\Delta \psi_{\mu}(x_1,\dots , x_n)\neq 0$ and $\nabla_\theta \psi_\mu(x_1,\dots, x_n) \neq 0$, see Lemma \ref{lem:perturbation}. Having established this, the lack of moments follows from asymptotically expanding the local energy $\frac{\mathcal{H}\psi_\mu}{\psi_\mu}$ and gradient random variable $\frac{\mathcal{H}\psi_\mu}{\psi_\mu}\cdot \frac{\nabla_\theta\psi_\mu}{\psi_\mu}$ around this zero, see Lemma \ref{lem:regularvalue}.
\begin{remark}
    Theorem \ref{thm:mainmomentslater} holds for any Hamiltonian. In  particular, the lack of moments is not caused by singular potentials such as \eqref{eq:BOHamiltonian}. Moreover it is possible to extend the result to general Slater-Jastrow Ans\"atze whenever the orbitals are chosen from a sufficiently expressive space, see Lemma \ref{lem:perturbation}.
\end{remark}
\begin{remark}At present, our main result applies to Slater-Jastrow Ans\"atze,
    but in recent years more general representations have become popular. For example, in \cite{hermann2020deep,renaud2025qmctorch} the Slater-Jastrow Ansatz  is generalized by applying $\psi_\theta$ to correlated variables $B_\theta(\mathbf{x})$ through a parametrized \emph{backflow transform}. A more general ansatz is given by FermiNet \cite{pfau2020ab} and subsequent work \cite{scherbela2022weightSharing,scherbela2023budget,scherbela2023foundationPreprint,scherbela2024transferableMolecules,scherbela2025fire,gerard2022goldStandard,gerard2024deepLearningVMCChapter,gerard2025transferableSolids,hermann2023abinitioReview,gao2022pesnet,gao2023generalizing,gao2023samplingFree}. We expect that the conclusion of Theorem \ref{thm:mainmomentslater} remains true in these more general models using similar methods as in the proof of Theorem \ref{thm:mainmomentslater}. We leave this to future work.  
\end{remark}
\subsubsection{Existence of Moments}
Having established a generic lack of high moments, our second main contribution consists of sharp lower bounds on the number of moments of $Z_\theta$: For real-analytic parametric function classes $\Psi$, we show in Theorem \ref{thm:momentexistencegrad} that (subject to additional compactness assumptions) a minimal number of moments always exists:

\begin{theorem}[Colloquial Version of Theorem \ref{thm:momentexistencegrad}]
    Let $K\subset \Omega$ compact with smooth boundary, and let $\eta\in C^\infty(\Omega)$ be any smooth cutoff function with $\mathrm{supp}(\eta) = K$. Then, for any parametrized family $\Theta \ni\theta\mapsto \psi_\theta:=\eta \varphi_\theta$ with $\varphi_\theta$ real-analytic and $\Theta$ compact, it holds that 
    \begin{equation*}
        \sup\{p\in (0,\infty):\ \sup_{\theta\in \Theta}  \mathbb{E}[|Z_{\theta}|^p]<\infty \}\geq \frac{5}{4}.
    \end{equation*}
\end{theorem}
The proof of Theorem \ref{thm:momentexistencegrad} makes use of the Weierstrass preparation theorem \cite[Section I.3.3]{range1998holomorphic}, along with an auxiliary non-degeneracy result in Lemma \ref{lem:orderalign} to carefully control the local structure of nodal sets of $\psi_\theta$.
A few remarks are in order.
\begin{remark}\label{rem:sharpmomentsgrad}
    By considering the function $\mathbf{x}\mapsto x_1^2$ it is easy to see that the threshold $p=\frac{5}{4}$  is sharp. 
\end{remark}
\begin{remark}\label{rem:GenericVSWorst}
    Suppose that $\R^P\ni\theta\mapsto \psi_\theta = \eta\varphi_\theta$ (with $\eta$ as in Theorem \ref{thm:momentexistencegrad}) is a parametrized family of compactly supported $C^2$ functions and that, for $\mu\in \R^P$, $0$ is a regular value of $\varphi_\mu$. Then, one can establish via first order Taylor expansion that $\mathbb{E}[|Z_\mu|^p]<\infty$ for every $p< \frac{3}{2}$, which is consistent with the lower bound from Theorem \ref{thm:mainmomentslater}. Since by Thom's transversality theorem \cite[Chapter 3, Theorem 2.1]{hirsch2012differential} a generic $C^2$ function has $0$ as a regular value, it follows -- roughly speaking -- that generically the variable $Z_\theta$ has $p$ moments for every $p<\frac{3}{2}$ but as soon as $\psi_\nu$ has higher order zeroes, the moment threshold drops to $p<\frac{5}{4}$. In other words, the \emph{generic} critical threshold is $\frac32$ but the \emph{worst-case} critical threshold is $\frac54$.
\end{remark}
\begin{remark}
    Using Malgrange's preparation theorem \cite{malgrange1964preparation}, one can extend Theorem \ref{thm:momentexistencegrad} to the case where the family $\varphi_\theta(\mathbf{x})$ is merely $C^\infty$ and not real-analytic, as long as the functions $\varphi_\theta$ do not posses "flat zeros" where all derivatives vanish. 
\end{remark}
\begin{remark}
    The assumption of $\Theta$ being compact in Theorem \ref{thm:momentexistencegrad} is needed to achieve a uniform bound of $\mathbb{E}[|Z_\theta|^p]$ over all parameters $\theta\in \Theta$. However, the optimization algorithms presented and analyzed in Section \ref{sec:optimization} assume unconstrained optimization over $\theta\in \R^P$ and therefore -- strictly speaking -- the families of Theorem \ref{thm:momentexistencegrad} do not fall into the unconstrained optimization framework studied in Section \ref{sec:optimization}. This is not really an issue, since one can easily extend the families $(\psi_\theta)_{\theta\in \Theta}$ to $\tilde \psi_\theta:=\psi_{g(\theta)}$, $\theta\in \R^P$ and $g:\R^P\to \Theta$ smooth with uniformly bounded derivatives. It is easy to see that this new family still satisfies the same moment assumptions and, in particular, the Assumptions \ref{set: SGD assumptions} needed in the convergence analysis of Section \ref{sec:optimization}. 
\end{remark}
\begin{remark}
    The result of Theorem \ref{thm:momentexistencegrad} requires the functions $\psi_\theta$ to be compactly supported, which is a restriction. However, the restriction is quite mild for applications related to Variational Monte Carlo, since it is known that bound states (e.g., eigenfunctions of $\mathcal{H}$ corresponding to eigenvalues below the ionization threshold) are exponentially decaying, see \cite[Theorem 8.6]{gustafson2003mathematical} or \cite{agmon2006bounds}. Extending Theorem \ref{thm:momentexistencegrad} to globally supported functions seems challenging; a potential avenue could be techniques similar to those used in \cite[Section 3]{grohs2021stable}. We leave this to future work.
\end{remark}
\begin{remark}
    Theorem \ref{thm:momentexistencegrad} does not yet apply fully to families $\psi_\theta$ with cusps, which by Kato's cusp conditions \cite{kato1957eigenfunctions} always appear in ground states of the multi-electron Hamiltonian. We do however expect that with similar arguments one can extend our results and leave the incorporation of cusps to future work. 
\end{remark}

\subsubsection{Robust Optimization}
Having clarified the precise regime of moments for the gradient random variable, we continue with our third main contribution: a new robust variant of VMC that provably converges to a stationary point of $L$, both in expectation (see Theorem \ref{thm: convergence in exp of vartiational monte carlo}) and with high probability (see Theorem \ref{thm: high prob conv of variational monte carlo}). This algorithm -- coined PS-Clip-VMC -- arises as a generalization of our recent work \cite{nobile2026robust} where robust gradient estimators are constructed using a \emph{per sample clipping} procedure. Roughly speaking, the resulting gradient estimators assumes $n$  i.i.d. samples $X_1,\dots , X_n\sim X_\theta$ and corresponding i.i.d. samples, $W_1,\dots , W_n\sim W_\theta$ and $E_1,\dots , E_n\sim E_\theta$ and approximates $\nabla L(\theta) = 2\mathbb{E}[Z_\theta]$ by a clipped empirical mean
$$
    \frac{2}{n} \sum_{k=1}^n \gamma_{\alpha,p,k}\left(E_kW_k\right)-
    \frac{2}{n^2-n} \sum_{k \neq j}\gamma_{\beta,2,k}\left(E_k\right)  \gamma_{\beta,2,k}\left(W_j\right),\ \mbox{where}\ 
    \gamma_{\alpha,p,k}(v) \coloneqq v\cdot \min \left\{1, \frac{\alpha k^{\frac{1}{p}}}{\left|v\right|} \right\}.
$$
In Theorem \ref{thm: convergence in exp of vartiational monte carlo} we show that under weak moment assumptions, consistent with the one proved in Section \ref{sec: Existence of Moments}, using this robust estimator guarantees convergence in expectation to a stationary point. Specifically, we show that for any small enough, constant step-size and any batch-size $n\geq M^{\frac{p}{2(p-1)}}$ the iterates of PS-Clip-VMC satisfy
\begin{equation*}
        \frac{1}{M}\sum_{m=1}^M\E(|\nabla_\theta L(\theta_m)|^2)= O\left(M^{-1}\right).
\end{equation*}
In Theorem \ref{thm: high prob conv of variational monte carlo} we then show that under the same assumptions, and for any $\delta \in (0,1)$ the iterates satisfy
\begin{equation*}
    \frac{1}{M}\sum_{m=1}^M|\nabla_\theta L(\theta_m)|^2 =O\left( \frac{\left({\log(4M/\delta)+1/4}\right)^{\frac{2(p-1)}{p}}}{M}\right)\;.
\end{equation*}
with probability $1-\delta$.

To the best of our knowledge, these results establish, for the first time, convergence of VMC to a stationary point under realistic moment assumptions.

\begin{remark}
In this work, we assume access to a stochastic oracle that provides i.i.d. samples from the distribution $|\psi_\theta|^2 / ||\psi_\theta||_{L^2}^2$. In practical VMC implementations, however, the samples are typically generated using Markov chain Monte Carlo (MCMC). The work \cite{li2024convergenceanalysisstochasticgradient} analyzes the convergence of stochastic gradient methods with MCMC sampling, but its analysis relies on substantially stronger moment assumptions than the ones considered in the present work. We leave the extension of our results to the practically relevant MCMC setting to future work.
\end{remark}

\begin{remark}
Clipping only the local energy observations is a common practice for stabilizing training in variational Monte Carlo \cite{hermann2023abinitioReview,hermann2020deep,gao2022pesnet, gao2023samplingFree, vonGlehn2023psiformer, gerard2022goldStandard}. In Appendix \ref{sec: In expectation convergence of Energy Clipping}, we show that this is sufficient to guarantee convergence in expectation, but not convergence with high probability in the heavy-tailed regime. In practice, local energy clipping is usually combined with a \emph{global} clipping, or norm constraint, of the stochastic gradient. It is possible that this combination is also sufficient to ensure convergence with high probability. However, the preliminary experiments in Section \ref{sec: Experiments} indicate that incorporating per-sample clipping of the 
gradient random variables can improve the stability of practical VMC training.
\end{remark}
\subsubsection{Experiments}
We complement our theoretical results with preliminary numerical experiments using FermiNet \cite{pfau2020ab,spencer2020better} applied to atoms with up to 18 electrons. 
Our results in Section \ref{sec:optimization} suggest that, in the heavy-tailed regime of VMC, 
one should clip both the local energies and the individual log-gradient contributions before averaging. The experiments in 
Section \ref{sec: Experiments} are intended to test this principle in a realistic setting. To this end, we train FermiNet on Sulfur and Argon, using the standard components of modern neural-network-based VMC such as MCMC sampling, 
KFAC updates, and centered local-energy clipping. In order to stay close to competitive practical VMC 
training, the algorithm used in the experiments is not a literal implementation of the estimator analyzed in 
Section \ref{sec:optimization}. In particular, instead of clipping thresholds which increase with the sample index, we use constant 
centered clipping thresholds, in analogy with the clipping rule commonly used for the local energy. The principle of per-sample clipping however remains. 

Our experiments indicate that per-sample gradient clipping can improve the stability and performance of VMC. In particular, using only half the number of samples per batch typically employed to train FermiNet, the model trained with PS-Clip-VMC matches the best published energy for Argon and improves upon the best published energy for Sulfur; see Table~\ref{tab: final energies}. The robustness of PS-Clip-VMC is particularly evident in the training trajectory for Argon shown in Figure~\ref{fig: Ferminet training trajectories}. The standard method exhibits a sharp increase in energy around step 55,000, and is not able to recover and converge to a lower energy after that.

\begin{figure}[h]
    \centering
    \includegraphics[width=0.45\textwidth]{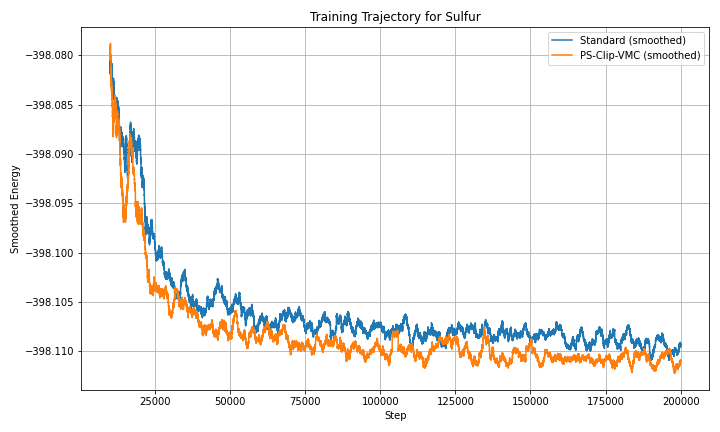}
    \includegraphics[width=0.45\textwidth]{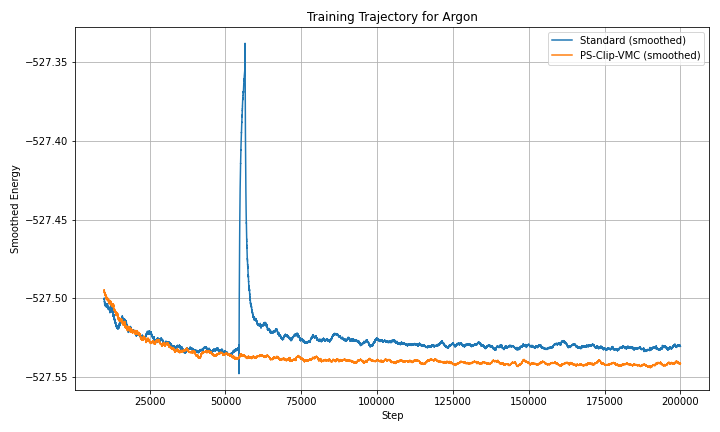}
    \hfill
    \caption{Training trajectories of Sulfur and Argon using a batchsize of 2048. The values are smoothed using a rolling mean of the last 2000 iterations.}
    \label{fig: Ferminet training trajectories}
\end{figure}
\begin{table}[h]
\centering
\begin{tabular}{c c c c}
\midrule
 Atom & {Standard} &{PS-Clip-VMC}&Best Published  \\

\midrule

S  & -398.1089(3)& \textbf{-398.1110(3)}&-398.110\textsuperscript{b} \\
Ar & -527.5304(4)&\textbf{-527.5420(3)}&-527.5419\textsuperscript{a} \\

\bottomrule
\end{tabular}
\caption{Final energies for Sulfur and Argon using a batch-size of 2048. The cited values are from a: DeepErwin \cite{gerard2022goldStandard}, b: CAS + experimental corrections \cite{chakravorty1993ground}.}
\label{tab: final energies}
\end{table}
\subsection{Related Work}
The literature on VMC is quite vast, see \cite{becca2017quantum} and references therein.  Especially for the Born-Oppenheimer multi-electron Hamiltonian the method has found renewed interest due to the empirical success of FermiNet and related methods, see \cite{pfau2020ab,hermann2020deep,renaud2025qmctorch,scherbela2023budget,scherbela2023foundationPreprint,scherbela2024transferableMolecules,scherbela2025fire,gerard2022goldStandard,gerard2025transferableSolids,hermann2023abinitioReview,gao2022pesnet,gao2023generalizing,gao2023samplingFree} and the survey paper \cite{gerard2024deepLearningVMCChapter}.
The mathematical study of these methods is still at its infancy. Important results concern the expressivity of Deep Learning VMC ans\"atze, see for example \cite{pang2022universalAntisymmetry,lin2023explicitly}. The optimization problem for VMC has been studied mathematically in \cite{abrahamsen2024convergence, li2024convergenceanalysisstochasticgradient}, albeit under moment conditions which -- as we show -- are not even satisfied for generic standard Slater-Jastrow ans\"atze. More recently, \cite{wan2026removingnodalsupportmismatchpathologies} proposed blurred sampling, a post-processing technique for mitigating the statistical pathologies caused by nodal singularities in VMC.

Regarding the existence of heavy tails, the most closely related work we found is the paper \cite{trail2008heavy}, where asymptotic expansions of the wavefunction and local energy are provided. Similarly to our first main result, Theorem \ref{thm:mainmomentslater}, it is concluded that the energy random variable $E_\theta$ is heavy-tailed. Therefore, to our knowledge, \cite{trail2008heavy} should be credited with first discovering the heavy-tailed nature of VMC algorithms. We note, however, that the main focus of \cite{trail2008heavy} is to establish CLT-like results in order to construct asymptotic confidence intervals. Moreover, we emphasize that the results in \cite{trail2008heavy} are non-rigorous and do not prove the validity of its asymptotic expansions (which would indeed require complicated arguments along the lines of our Lemma \ref{lem:perturbation}). Finally, unlike our work, \cite{trail2008heavy} does not establish the existence of a certain number of moments or any rigorous convergence theorems.

\subsection{Notation }
For $K\in \N$ we write $[K]:=\{1,\dots , K\}$. Moreover for $p\in [1,\infty]$, $k,d\in \N$ and $\Omega\subset \R^d$ with smooth boundary we write $L^p(\Omega)$, $H^k(\Omega)$, $C^k(\Omega)$, and $L^p_{\mathrm{loc}}(\Omega)$ for the Lebesque space with integrability $p$, the $k$-th order Sobolev space, the space of $k$-times continuously differentiable functions, and the space of locally $p$-integrable functions. For $N,d\in \N$ and $\mathbf{x}=(x_1,\dots , x_N)\in \R^{d\times N}$ we occasionally write $\hat{\mathbf{x}}_i:=(x_1,\dots , x_{i-1},x_{i+1},\dots , x_N)\in \R^{d\times (N-1)}$.
\subsection{Outline}
The outline is as follows. In Section \ref{sec:nonexistencemoments} we prove the non-existence of moments for general Slater-Jastrow ans\"atzte. In Section \ref{sec: Existence of Moments} we prove the existence of a minimal number of moments for real-analytic ans\"atze. 
In Section \ref{sec:optimization} we introduce and analyze novel optimization algorithms and prove their convergence. Finally, in Section \ref{sec: Experiments} we present computational experiments. 
\section{Non-existence of Moments for a Slater-Jastrow Ansatz}
\label{sec:nonexistencemoments}
The present section is devoted to proving that the random variables $E_\theta$ for the energy and $Z_\theta$ for the energy gradient are generically heavy tailed. 
\subsection{Formulation of the Main Result}
We start with describing the main settings of this section, starting with our main notation and assumptions for hamiltonians and associated random quantities.  
\begin{setting}[Hamiltonian]\label{set:Hamiltonian}
    For $d\in \N$, $\Omega\subset \R^d$ open and $\psi\in H^2(\Omega)$ let $\mathcal{V}\in L^1_{\mathrm{loc}}(\Omega)$ be such that the \emph{Hamiltonian} $$\mathcal{H}:\left\{ \begin{array}{ccc}H^2(\Omega)&\to &L^2(\Omega)\\ \psi& \mapsto & -\Delta \psi + \mathcal{V}\psi\end{array}\right.$$ is bounded
    (see for example \cite[Theorem 4.3]{kato:76:perturbation} for suitable conditions for this to hold). 
    
    For $\mathbf{x}\in \R^d$ and $\psi\in H^2(\Omega)$ we define the  \emph{local energy} $E_{L,\psi}(\mathbf{x}):=\frac{\mathcal{H}\psi(\mathbf{x})}{\psi(\mathbf{x})}$ and the \emph{density} $p_\psi(\mathbf{x}):=\frac{|\psi(\mathbf{x})|^2}{\|\psi\|^2_{L^2}}$.
    For $X_\psi\sim p_\psi$ we call the random variable $E_\psi:=E_{L,\psi}(X_\psi)$ the \emph{energy random variable}. It satisfies that $\mathbb{E}[E_\psi]= \frac{1}{\|\psi\|^2_{L^2}}\langle \mathcal{H}\psi,\psi\rangle_{L^2}$.
    
    If $\R^P\ni \theta \mapsto \psi_\theta\in H^2(\Omega)$ is a parametrized model such that the map $\theta\mapsto \psi_\theta\in L^2(\Omega)$ is differentiable, we denote $p_\theta:=p_{\psi_\theta}$, $X_\theta:=X_{\psi_\theta}$, $E_{L,\theta}:=E_{L,\psi_\theta}$, $E_\theta:=E_{\psi_\theta}$, and $W_\theta:=\frac{\nabla_\theta\psi_\theta(X_\theta)}{\psi_\theta(X_\theta)}$.
    We call the random variable $$Z_\theta:=  E_\theta \cdot W_\theta-\mathbb{E}[E_\theta]\cdot \mathbb{E}\left[W_\theta\right]$$ the \emph{energy gradient random variable}. It satisfies that 
    $$2\mathbb{E}[Z_\theta] = \nabla_\theta \mathbb{E}[E_\theta] \quad \mbox{and}\quad \mathbb{E}\left[W_\theta\right] = \frac12\frac{\nabla_\theta \|\psi_\theta\|_{L^2}^2}{\|\psi_\theta\|_{L^2}^2},$$ see \cite[Section 2.2]{armegioiu2025functional}. 

     %For $A\subset \R^d$ open we denote with $C^\omega(A)$ the set of real-analytic functions on $A$.
\end{setting}
The present section focuses on fermionic many-particle systems which are formally introduced in the following setting.
\begin{setting}[Fermionic Many-Particle Systems and Slater Determinants]\label{set:HF}
    Assume Setting \ref{set:Hamiltonian}. For $N\in \N$, $A\subset \R^3$ and $F(A^N)$ a function space on $A^N$ (such as $L^2(A^N)$ or $C(A^N)$) we denote
    $$
        F_a\left(A^N\right):=\left\{\psi\in F\left(A^N\right):\ \forall \sigma \in \mathfrak{S}_N,\ \mathbf{x}=(x_1,\dots , x_N)\in A^N:\ \psi(x_{\sigma(1)},\dots , x_{\sigma(N)}) = (-1)^{\mathrm{sgn}(\sigma)}\psi(x_1,\dots , x_N)\right\},
    $$
    where $\mathfrak{S}_N$ denotes the set or permutations of $[N]$ and for a permutation $\sigma\in \mathfrak{S}_N$ we denote with $\mathrm{sgn}(\sigma)$ its signature. If in Setting \ref{set:Hamiltonian} it holds that $\Omega = A^N$ and $\mathcal{H}$ is restricted to $H^2_a(\Omega)$, we speak of a \emph{fermionic many-particle system}.
    
    For $h_1,\dots , h_N\in C(\R^3)$ we define the \emph{Slater determinant} $|h_1,\dots , h_N\rangle \in C_a(\R^{3\times N})$ via 
    \begin{equation}\label{eq:SDdef}
        |h_1,\dots , h_N\rangle(x_1,\dots, x_N):=\det\begin{pmatrix}h_1(x_1)& \dots & h_N(x_1)\\\vdots & \ddots & \vdots \\h_1(x_N) & \dots & h_N(x_N)\end{pmatrix},\quad  (x_1,\dots , x_N)\in \R^{3\times N}.
    \end{equation}
    For $M\in \N$ and $\mathbf{R}=(R_1,\dots , R_M)\in \R^{3\times M}$ define the \emph{collision set}
    $$
        \Sigma(\mathbf{R}):=\{(x_1,\dots , x_N)\in \R^{3\times N}:\ \exists i\in [N],\ j\in [N]\setminus \{i\}:\ x_i = x_j\}\cup \bigcup_{i\in [N]}(\R^3)^{i-1}\times \bigcup_{j\in [M]}\{R_j\}\times (\R^3)^{N-i}
    $$
    and assume that $\mathcal{V}\in C(\R^{3\times N}\setminus \Sigma(\mathbf{R}))$.
    
\end{setting}
Finally, we introduce the Slater-Jastrow ansatz for the numerical solution of \eqref{eq:RayleighRitzElectronic}. 
\begin{setting}[Slater-Jastrow Ansatz with variable exponent Slater-Type Orbitals (STOs)]\label{set:slaterorbs}
    Assume Settings \ref{set:Hamiltonian} and \ref{set:HF}.
    For $l\in \N_0$ let $Y_{l,-l},\dots , Y_{l,l}$ be \emph{spherical harmonics}, e.g., an ONB of 
    $$
        \mathfrak{H}_l:=\{u\in L^2(\mathbb{S}^2):\ u\ \textrm{harmonic homogeneous polynomial of degree }l\ \}\subset L^2(\mathbb{S}^2).
    $$
    Given $L,I,K\in \N$, $\alpha_1,\dots , \alpha_I\in [0,\infty)$ the corresponding \emph{Slater-type orbitals (STOs) with variable exponents} (see for example \cite[Section 6.5.6]{helgaker2013molecular}) are defined as
    \begin{equation}\label{eq:STOdef}
        \R^3\ni x \mapsto h_{\alpha_i,l,m,k}(x):=Y_{l,m}\left(\frac{x}{|x|}\right)|x|^{l+k}e^{-\alpha_i|x|},\quad l=0,\dots,L,\ m=-l,\dots , l,\ k=0,\dots, K-1,\ i=1,\dots , I.
    \end{equation}
    For $\bm{\alpha}\in [0,\infty)^I$, $\mathbf{c} =\left(c_{l,m,k,i}\right)_{l=0,\dots,L,\ m=-l,\dots , l,\ k=0,\dots, K-1,\ i=1,\dots , I}\in \R^{(L+1)^2KI}$ and $x\in \mathbb{R}^3$ we denote 
    $$
        h_{(\bm{\alpha},\mathbf{c})}(x):=\sum_{l,m,k,i}c_{l,m,k,i}h_{\alpha_i,l,m,k}(x).
    $$
    Let  $\Sigma:=\Sigma(0)$
    and $B,J,N\in \N$ and $\R^J\ni \beta\mapsto \mathcal{J}_\beta\in C^2(\R^{3\times N}\setminus \Sigma)$ smooth with $\mathcal{J}_\beta$ symmetric and positive. 
    Then, for $((\bm{\alpha}_a^b,\mathbf{c}_a^b))_{a=1,\dots , N,\ b=1,\dots , B}\in \left([0,\infty)^{I}\times \R^{(L+1)^2KL}\right)^{NB}\cong[0,\infty)^{NBI}\times \R^{NB(L+1)^2KL}$ and $\theta:=\left(\beta,((\bm{\alpha}_a^b,\mathbf{c}_a^b))_{a=1,\dots , N,\ b=1,\dots , B}\right)\in \R^J\times [0,\infty)^{NBI}\times \R^{NB(L+1)^2KL}$ we denote
    $$
        \psi_\theta(\mathbf{x})\coloneqq\mathcal{J}_\beta(\mathbf{x})\cdot\left(\sum_{b=1}^B |h_{(\bm{\alpha}_1^b,\mathbf{c}_1^b)},\dots ,h_{(\bm{\alpha}_N^b,\mathbf{c}_N^b)}\rangle(\mathbf{x}) \right)
    $$
    as the Slater-Jastrow ansatz with variable exponent Slater-type orbitals.
\end{setting}
We can now rigorously formulate our first main theorem. 
\begin{theorem}\label{thm:mainmomentslater}
    Assume Settings \ref{set:Hamiltonian} and \ref{set:slaterorbs}. Assume further that $L\geq 1$ and $KI\geq 3N-3$. 
    Then, the set
    \begin{equation}
        \left\{\theta\in \R^J\times [0,\infty)^{NBI}\times \R^{NB(L+1)^2KL}: \mathbb{E}\left[|E_{\theta}|^3\right]=\infty \quad \land \quad \mathbb{E}\left[|Z_\theta|^{3/2}\right]=\infty\right\} 
    \end{equation}
    contains a dense open subset in $\R^J\times [0,\infty)^{NBI}\times \R^{NB(L+1)^2KL}$. In other words, nonexistence of a $3$rd moment of $E_\theta$ and of a $3/2$-th moment of $Z_\theta$ constitutes a \emph{generic property}. 
\end{theorem}
\subsection{Proof of the Main Result}
We now turn to the proof of Theorem \ref{thm:mainmomentslater}. At a high level we will show that every $\psi_\theta$ has an arbitrarily small perturbation with certain regular zeros. 
Then, the existence of such zeros implies a lack of moments, as the following elementary results show. 
\begin{lemma}\label{lem:regularvalue}
    Assume Setting \ref{set:Hamiltonian}, let $A\subset \Omega$ open with $\mathcal{V}\in C(A)$ and $\psi\in C^2(A)$. Assume there exists $\mathbf{x}\in A$ so that
    \begin{equation}\label{eq:regularvalue}
        \psi(\mathbf{x})=0,\quad \nabla \psi(\mathbf{x})\neq 0,\quad \mbox{and}\quad \Delta\psi(\mathbf{x})\neq 0.
    \end{equation}
    Then, it holds that 
    \begin{equation}\label{eq:momentinfinity}
        \mathbb{E}\left[|E_\psi|^3\right]=\infty.
    \end{equation}
    If $\psi = \psi_\theta$ and additionally $\nabla_\theta \psi_\theta(\mathbf{x})\neq 0$ with $\nabla_\theta\psi_\theta$ continuous at $\mathbf{x}$, then it also holds that
    \begin{equation}\label{eq:gradientmomentinfinity}
        \mathbb{E}\left[|Z_\theta|^{3/2}\right]=\infty \;.
    \end{equation}
\end{lemma}
\begin{proof}  
Let $d = 3n$ and $\mathbf{x}\in A$ with $\psi(\mathbf{x})=0$
and $\nabla\psi(\mathbf{x})\neq 0$ and $\Delta \psi(\mathbf{x})\neq 0$. 
For $\mathbf{y}\in \R^d$ we denote $\mathbf{y}_{d-1}:=(y_1,\dots , y_{d-1})$.
Without loss of generality, suppose that $\frac{\partial}{\partial x_d}\psi(\mathbf{x})\neq 0$ and $\|\psi\|_{L^2(\Omega)}=1$.
Then, by the implicit function theorem there exists an open set $U\subset \R^{d-1}$ with $\mathbf{x}_{d-1}\in U$ and $g\in C^1(U)$ with 
\begin{equation}\label{eq:implicitfunction}
    g(\mathbf{x}_{d-1}) = x_d\quad \land \quad \forall \,\mathbf{y}_{d-1}\in U:\ \psi(\mathbf{y}_{d-1},g(\mathbf{y}_{d-1}))=0.
\end{equation}
For $\varepsilon>0$ denote
$$
    B_\varepsilon\coloneqq\{ (\mathbf{y}_{d-1},g(\mathbf{y}_{d-1})+\tau):\ \mathbf{y}_{d-1}\in U,\ \tau\in (-\varepsilon,\varepsilon)\}\subset \R^{d}.
$$
Then, $B_\varepsilon$ is open and for every integrable $\varphi:B_\varepsilon \to \mathbb{R}$ it holds that
$$
    \int_{B_\varepsilon}\varphi(\mathbf{y}) d\mathbf{y}= \int_{-\varepsilon}^\varepsilon \int_U\varphi(\mathbf{y}_{d-1},g(\mathbf{y}_{d-1})+\tau)d\mathbf{y}_{d-1}d\tau.
$$
By noting that $\mathcal{V}(\mathbf{x})\psi(\mathbf{x})=0$ and shrinking $U$ and $\varepsilon$ if necessary, we can ensure that 
$B_\varepsilon\subset A$ and
\begin{equation}\label{eq:lowerH}
    \inf_{\mathbf{y}\in B_\varepsilon}|\mathcal{H}\psi(\mathbf{y})|=:\delta>0.
\end{equation}
Now, note that 
\begin{align}
     \mathbb{E}\left[|E_\psi|^3\right]&=\int_{\Omega}\left|\frac{\mathcal{H}\psi(\mathbf{y})}{\psi(\mathbf{y})}\right|^3 |\psi(\mathbf{y})|^2d\mathbf{y} \\ 
     &\geq
     \delta^3\int_{B_\varepsilon}|\psi(\mathbf{y})|^{-1}d\mathbf{y} \\
     &= 
     \delta^3 \int_{-\varepsilon}^\varepsilon \int_U |\psi(\mathbf{y}_{d-1},g(\mathbf{y}_{d-1})+\tau)|^{-1}d\mathbf{y}_{d-1}d\tau.
     \label{eq:intlower}
\end{align}
Furthermore, using \eqref{eq:implicitfunction}, it holds for $\mathbf{y}_{d-1}\in U$ and $\tau\in (-\varepsilon,\varepsilon)$ that
\begin{align*}
    \psi(\mathbf{y}_{d-1},g(\mathbf{y}_{d-1})+\tau) &=\psi(\mathbf{y}_{d-1},g(\mathbf{y}_{d-1}))+\tau \cdot \frac{\partial}{\partial x_d}\psi(\mathbf{y}_{d-1},g(\mathbf{y}_{d-1})) + \eta(\mathbf{y}_{d-1},\tau) \\
    &= \tau \cdot \frac{\partial}{\partial x_d}\psi(\mathbf{y}_{d-1},g(\mathbf{y}_{d-1})) + \eta(\mathbf{y}_{d-1},\tau)
\end{align*}
with 
$$
    |\eta(\mathbf{y}_{d-1},\tau)|\le \tau^2 \|\psi\|_{C^2(B_\varepsilon)}.
$$
Therefore, for all $\mathbf{y}_{d-1}\in U$ and $\tau\in (-\varepsilon,\varepsilon)$ we have that
$$
    \left|\psi(\mathbf{y}_{d-1},g(\mathbf{y}_{d-1})+\tau)\right|\le |\tau|(1+\varepsilon)\|\psi\|_{C^2(B_\varepsilon)}\;.
$$
Inserting this inequality into \eqref{eq:intlower} yields that
\begin{equation*}
    \mathbb{E}\left[|E_\psi|^3\right]\geq   \frac{|U|\delta^3}{(1+\varepsilon)\|\psi\|_{C^2(B_\varepsilon)}} \int_{-\varepsilon}^\varepsilon \frac{1}{|\tau|}d\tau = \infty ,
\end{equation*}
which proves \eqref{eq:momentinfinity}.

To prove \eqref{eq:gradientmomentinfinity}, we can apply the same reasoning to the integral
\begin{equation}
    \int_{B_\varepsilon}\left|\frac{\mathcal{H}\psi(\mathbf{y})}{\psi(\mathbf{y})}\frac{|\nabla_\theta\psi(\mathbf{y})|}{\psi(\mathbf{y})}\right|^{3/2} |\psi(\mathbf{y})|^2d\mathbf{y} \;,
\end{equation}
and use the fact that (by the assumption that $\nabla_\theta\psi_\theta(\mathbf{x})\neq 0$ and shrinking $U,\varepsilon$ if needed) it holds that $|\nabla_\theta\psi(\mathbf{y})|>0$ for all $\mathbf{y}\in B_\varepsilon$. 
This implies that
\begin{equation}\label{eq:Zdominant}
    \int_{B_\varepsilon}\left|\frac{\mathcal{H}\psi(\mathbf{y})}{\psi(\mathbf{y})}\frac{|\nabla_\theta\psi(\mathbf{y})|}{\psi(\mathbf{y})}\right|^{3/2} |\psi(\mathbf{y})|^2d\mathbf{y} = \infty .
\end{equation}
Since \eqref{eq:Zdominant} is the dominant term in $Z_\theta$,  it follows that $\mathbb{E}[|Z_\theta|^{3/2}]=\infty$.  
\end{proof}
\begin{lemma}\label{lem:regularopen}
    Let $A\subset \R^{3\times N}$ open.  Then,
    $$
        \mathcal{O}:=\{\psi\in C^2(A):\ \exists \mathbf{x}\in A:\  \psi(\mathbf{x})=0,\  \nabla \psi(\mathbf{x})\neq 0,\  \mbox{and}\  \Delta\psi(\mathbf{x})\neq 0\}
    $$
    is open in the local (compact-open) topology of $C^2(A)$.

    If $\psi = \psi_\theta$ such that the mapping $\R^P\ni\theta\mapsto (\psi_\theta,\nabla_\theta\psi_\theta)\in C^2(A)\times C(A,\R^P)$ is continuous in the local (compact-open) topology, then the set   
    $$
        \mathcal{O}':=\{\theta\in \R^P:\ \exists \mathbf{x}\in A:\  \psi_\theta(\mathbf{x})=0,\  \nabla \psi_\theta(\mathbf{x})\neq 0,\   \Delta\psi_\theta(\mathbf{x})\neq 0, \mbox{and}\  \nabla_\theta\psi_\theta(\mathbf{x})\neq 0\}
    $$
    is open.
\end{lemma}
\begin{proof}
    Recall that in the local $C^2$ topology, for every $\psi\in C^2(A)$, every compact $B\subset A$ and for every $\varepsilon>0$, the set $V(B,\varepsilon,\psi):=\{\varphi\in C^2(A):\ \|\varphi - \psi\|_{C^2(B)}< \varepsilon\}$ is open. We will show that for every $\psi\in \mathcal{O}$ there is $B,\varepsilon$ such that $V(B,\varepsilon,\psi)\subseteq \mathcal{O}$.

    Let $\psi\in \mathcal{O}$ with $\mathbf{x}$ as in the definition of $\mathcal{O}$. Let $B\subset A$ be a $\tau$-neighborhood of $\mathbf{x}$ and let $\mathbf{y}\in \R^{3\times n}$ be a unit vector with 
    $$
        s:=\mathbf{y}\cdot \nabla \psi(\mathbf{x})>0.
    $$
    Then, it holds for all $t\in [-\tau,\tau]$ that 
    $$
        \varphi(\mathbf{x}+t\mathbf{y})=\varphi(\mathbf{x}) + ts + \eta_t,
    $$
    where $|\eta_t|\le \|\varphi\|_{C^2(B)}\cdot t^2$. By choosing 
    $$
        \tau<\frac{s}{\|\psi\|_{C^2(B)}}\quad \mbox{and}\quad \varepsilon <\frac{\tau s - \|\psi\|_{C^2(B)}\tau^2}{1+\tau+\tau^2}
    $$
    it is easy to see that, whenever $\varphi \in V(B,\varepsilon,\psi)$, we have that 
    $$
        \varphi(\mathbf{x} - \tau\mathbf{y})<0\quad \mbox{and}\quad \varphi(\mathbf{x} + \tau\mathbf{y})>0
    $$
    and therefore, there is $\tau_\ast\in (-\tau,\tau)$ such that with $\mathbf{x}_\ast:=\mathbf{x}+\tau_\ast\mathbf{y}\in B$ it holds that 
    $$
        \varphi(\mathbf{x}_\ast) = 0.
    $$
    By reducing $\tau,\varepsilon$ further if needed and using that $\nabla \psi(\mathbf{x})\neq 0$ and $ \Delta\psi(\mathbf{x})\neq 0$ we can ensure that 
    $$
        \nabla \varphi(\mathbf{x}_\ast)\neq 0\quad \mbox{and}\quad \Delta \varphi(\mathbf{x}_\ast)\neq 0. 
    $$
    This proves that $V(B,\varepsilon,\psi)\subset \mathcal{O}$ as needed.
    The statment on openness of $\mathcal{O}'$ follows now immediately from the continuity of the mapping $\theta\mapsto (\psi_\theta,\nabla_\theta\psi_\theta)$.
\end{proof}
The key difficulty in the proof of Theorem \ref{thm:mainmomentslater} is now to show that any $\psi_\theta$ can be slightly perturbed so that a zero as in Lemma \ref{lem:regularvalue} exists. We will ultimately achieve this by slightly changing the underlying Slater-type orbital, but our results will require sufficient expressivity of the underlying Slater-type orbitals. A key condition in this context is the following.  
\begin{definition}
    Let $V_r\subset C^2(0,\infty)$ be such that for $m\in \N$, every $0<\tau_1<\dots <\tau_m$ and every $(\nu_{j}^{(i)})_{j\in [m],\ i\in \{0,1,2\}}$ there is $v\in V_r$ with $\frac{d^i}{dt^i}v(\tau_j) = \nu_{j}^{(i)}$ for all $j\in [m],\ i\in \{0,1,2\}$. Then, $V_r$ is said to satisfy the $m$-Hermite interpolation property ($m$-HIP).
\end{definition}
We also mention that the $m$-HIP is closely related to $V_r$ constituting an extended Chebychev system \cite{pinkus2012n} -- a property that is known to be satisfied by exponentially weighted polynomials, as the following lemma shows. We emphasize however, that the $m$-HIP only applies to the radial part of a $3$-dimensional orbital. 
\begin{lemma}\label{lem:chebsys}
    If $V_r = \mathrm{span}\{r^k e^{\alpha_i r}:\ i=1,\dots I, k=0,\dots , m_i-1\}$ with $\alpha_i\in \R$ pairwise disjoint and $\sum_{i=1}^I m_i\geq 3m$, then $V_r$ satisfies the $m$-HIP.
\end{lemma}
\begin{proof}
    Observe that $V_r=\mbox{ker}\left(\left(\frac{d}{dt}-\alpha_1 \right)^{m_1}\dots \left(\frac{d}{dt}-\alpha_I \right)^{m_I}\right)$. 
    By \cite[Theorem 9]{aldaz2009bernstein}, the space $V_r$ is an extended Chebychev system in the sense of \cite[Definition III.1.8]{pinkus2012n}, which implies the $m$-HIP.
\end{proof}
The following Lemma establishes that the $(N-1)$-HIP implies, for specific nodes $x_i\in \R^3$, $i\in [N]$, that a multiplicative linear factor allows the construction of $3$-dimensional orbitals with prescribed point values and derivatives. We mention that such a result is by no means obvious: Due to the well-known Mairhuber-Curtis Theorem \cite{wendland2004scattered}, general interpolation properties of linear subspaces for $d$-dimensional functions are in general highly nontrivial for $d\geq 2$. 
\begin{lemma}\label{lem:interpolation}
    Let $N\in \N$, suppose that $V_r\subset C^2(0,\infty)$ satisfies the $(N-1)$-HIP and let 
    \begin{equation*}
        V:=\{\R^3\ni x\mapsto (w\cdot x)v_r(|x|):\ w\in \R^3\land v_r\in V_r\}.
    \end{equation*}
    Then, for every pairwise distinct $x_1,\dots , x_N\in \R^3\setminus \{0\}$ with $x_1,x_2$ non-collinear such that
    \begin{equation}\label{eq:noequalabs}
        x_{1,1}\neq 0\quad \land \quad |x_1|=|x_2|\quad \land \quad \forall i\in \{3,\dots , N\},\ j\in [N]\setminus \{i\}:\ |x_i|\neq |x_j| \;,
    \end{equation}
    and for every $\nu^{(1)},\nu^{(2)}\in \R$
    there exists $v\in V$ with 
    \begin{enumerate}[label=(\roman*), ref=(\roman*)]
        \item $v(x_i) = 0$ for all $i\in [N]$, \label{item:interpol}
        \item $\nabla v(x_1)=\frac{\nu^{(1)}}{x_{1,1}}x_1$ and $\nabla v(x_i) = 0$ for $i=\{2,\dots , N\}$, and \label{item:interpolder}
        \item $\Delta v(x_1)= \nu^{(2)}$ and $\Delta v(x_i)=0$ for $i=\{2,\dots , N\}$.\label{item:interpollap}
    \end{enumerate}
\end{lemma}
\begin{proof}
    Observe that for any $v(x)=(w\cdot x) v_r(|x|)$ with $v_r\in V_r$ and $k\in \{1,2,3\}$, it holds that
    \begin{equation}\label{eq:productrule}
        \partial_k v(x) = w_kv_r(|x|) + (w\cdot x)v_r'(|x|)\frac{x_k}{|x|}\quad \mbox{and}\quad \Delta v(x) = (w\cdot x)\left(v_r''(|x|) + \frac{4}{|x|}v_r'(|x|)\right).
    \end{equation}
    Now, pick $v_r\in V_r$ and $w\in \R^3$ so that
    \begin{equation}\label{eq:interpolconditions}
        \forall i\in [N]:\ v_r(|x_i|) = 0\ \land \ 
        \forall i\in \{3,\dots , N\}:\ v_r'(|x_i|) = v_r''(|x_i|) = 0\  \land \  w\cdot x_1 = :\tau\neq 0\  \land \ w\cdot x_2=0.
    \end{equation}
    Note that this is possible due to the $(N-1)$-HIP, and we are still free to choose values for $v_r'(|x_1|),\ v_r''(|x_1|)$. 
    Using \eqref{eq:productrule} and \eqref{eq:interpolconditions}, it is easy to check that $v(x):=(w\cdot x)v_r(|x|)$ satisfies
    $$
        \forall i\in [N]:\ v(x_i) = 0 \quad \land \quad \forall i\in \{2,\dots , N\}:\ \Delta v(x_i) = 0\ \land \ \nabla v(x_i) = 0.
    $$
    Furthermore, it holds that
    $$
        \partial_k v(x_1)=\tau v_r'(|x_{1}|)\frac{x_{1,k}}{|x_1|}\quad \mbox{and}\quad \Delta v(x_1) = \tau\left(v_r''(|x_1|) + \frac{4}{|x_1|}v_r'(|x_1|)\right).
    $$
    By choosing the values $v_r'(|x_1|), v_r''(|x_1|)$ appropriately, we can ensure that \ref{item:interpol}, \ref{item:interpolder} and \ref{item:interpollap} are satisfied. 
\end{proof}
The following key lemma constructs interpolation nodes which satisfy three properties simultaneously: the assumptions of Lemma \ref{lem:interpolation}, being a zero of a given antisymmetric function $\psi$, as well as a non-vanishing property of an associated Slater determinant. The proof uses a topological argument similar to the one in the proof of the Mairhuber-Curtis theorem.
\begin{lemma}\label{lem:existenceofzeros}
    Assume Setting \ref{set:HF} and let $v_2,\dots , v_N\in C(\R^3)$ linearly independent and $\psi\in C_a(\R^{3\times N})$. 
        Then, for $\mathbf{R}=(R_1,\dots , R_m)\in \R^{3\times M}$ there exist  $x_1,\dots , x_N\in \R^{3\times N}\setminus \Sigma(\mathbf{R})$ with $x_1,x_2$ non-collinear and
    \begin{equation}\label{eq:noequalabsagain}
        x_{1,1}\neq 0\quad \land \quad |x_1|=|x_2|\quad \land \quad \forall i\in \{3,\dots , N\},\ j\in [N]\setminus \{i\}:\ |x_i|\neq |x_j|
    \end{equation}
    such that 
    \begin{enumerate}[label=(\roman*), ref=(\roman*)]
        \item $\psi(x_1,\dots , x_N) = 0$, and
        \item $|v_2,\dots , v_N\rangle(x_2,\dots , x_N)\neq 0$.
    \end{enumerate}
\end{lemma}
\begin{proof}
    Since $v_2,\dots , v_N$ are linearly independent, there exist
    $\mathbf{z}=(z_\ast , z_3,\dots , z_N)\in \R^{3\times (N-1)}$ so that $$|v_2,\dots , v_N\rangle(z_\ast,z_3,\dots , z_N)\neq 0.$$  Due to continuity, there exists an open neighborhood $U$ of $\mathbf{z}$ in $\R^{3\times (N-1)}$ so that $|v_1,\dots , v_N\rangle$ vanishes nowhere on $U$. In particular, we can perturb $\mathbf{z}$ so that 
    $$
        |z_i|\neq |z_j|\quad \forall i\in \{\ast, 3,\dots , N\},\ j\in \{\ast, 3,\dots , N\}\setminus\{i\}\mbox{ and } |z_i|\notin \bigcup_{l=1}^M\{|R_l|\} \quad \forall i  \in \{\ast, 3,\dots , N\},
    $$
    as well as $z_{\ast,1}\neq 0$.
    Now, pick a circle $C=\{\gamma(t) = z_\ast + \varepsilon\sin(t)a + \varepsilon\cos(t)b : t\in [0,2\pi)\}$ with unit vectors $a,b \in \{z_\ast\}^\bot$ and $C\times \{z_3\}\times\dots \{z_N\}\subset U$.
    By making $\varepsilon$ sufficiently small, and using that $z_{\ast,1}\neq 0$, we can ensure that 
    $$
        \forall x\in C:\ x_1\neq 0\quad \mbox{and}\quad \forall x,y\in C:\ (x\neq y)\Rightarrow x,y\textrm{ not collinear} .
    $$
    Note that
    \begin{equation}\label{eq:samenorm}
            \forall t\in [0,2\pi):\ |\gamma(t)|^2 = |z_\ast|^2 + \varepsilon^2.  
    \end{equation}
    Now, consider the function 
    $$
        \tau(t):=\psi(\gamma(t),\gamma(t+\pi),z_3,\dots , z_N\rangle
    $$
    Due to antisymmetry of $\psi$, it holds that $\tau(0) = -\tau(\pi)$. Since $\tau$ is continuous, there must exist $t_0\in [0,\pi)$ such that 
    $$
        \psi (\gamma(t_0),\gamma(t_0+\pi),z_3,\dots , z_N)=0. 
    $$
    In summary, $(x_1,x_2,x_3,\dots , x_N) = (\gamma(t_0),\gamma(t_0+\pi),z_3,\dots , z_N)$ satisfies the claimed properties. 
\end{proof}
The following lemma constitutes the main technical result of this section and establishes the existence of perturbations with regular zeros. 
\begin{lemma}\label{lem:perturbation}
    Assume Settings \ref{set:HF} and \ref{set:slaterorbs}. Suppose that $V_r\subset C^2(0,\infty)$ satisfies the $(N-1)$-HIP and let $V$ be a vector space with
    \begin{equation*}
        V\supseteq\{\R^3\ni x\mapsto (w\cdot x)v_r(|x|):\ w\in \R^3\land v_r\in V_r\}.
    \end{equation*}
    Then, for any linearly independent functions $v_1,\dots , v_N\in C(\R^3)$ with $v_1\in V$, any $F\in C^2_a(\R^{3\times N}\setminus \Sigma)$, any symmetric $\mathcal{J}\in C^2(\R^{3\times N}\setminus \Sigma)$ with $\mathcal{J}>0$, there exist $\mathbf{x}\in \R^{3\times N}\setminus \Sigma$, $v_\delta  , u\in V\setminus\{0\}$ and $\lambda_0>0$ such that for all $\lambda \in (0,\lambda_0) $ and $\psi_{\delta,\lambda}:=F+\mathcal{J}\cdot|v_1+\lambda v_\delta,v_2,\dots , v_N\rangle$ we have that
    \begin{equation}\label{eq:regularvalueagain}
         \psi_{\delta,\lambda}(\mathbf{x})=0,\quad \nabla \psi_{\delta,\lambda}(\mathbf{x})\neq 0,\quad \mbox{and}\quad \Delta\psi_{\delta,\lambda}(\mathbf{x})\neq 0
    \end{equation}
    and
    \begin{equation}\label{eq:nonzerograd}
        \lim_{h\to 0}\frac{1}{h}\left(F(\mathbf{x})+\mathcal{J}(\mathbf{x})\cdot |v_1+\lambda v_\delta + hu,v_2,\dots , v_n\rangle(\mathbf{x}) - \psi_{\delta,\lambda}(\mathbf{x})\right) = \mathcal{J}(\mathbf{x})\cdot |u,v_2,\dots , v_n\rangle(\mathbf{x})\neq 0.
    \end{equation}
\end{lemma}
\begin{proof}
    Write $\psi:=F+\mathcal{J}\cdot |v_1,\dots , v_N\rangle$ and $\psi_\delta:=\psi_{\delta,1}$.
    For $i\in [N]$ and $(x_1,\dots , x_N)\in\R^{3\times N}$ let $\hat x_i:=(x_1,\dots , x_{i-1},x_{i+1},\dots , x_N)\in \R^{3\times (N-1)}$. Then, for all $\mathbf{x}=(x_1,\dots , x_N)\in \R^{3\times N}$ it holds that
    \begin{align}
        \psi_\delta(\mathbf{x}) &= \psi(\mathbf{x}) + \mathcal{J}(\mathbf{x})\sum_{i=1}^N(-1)^{i+1} v_\delta(x_i)|v_2,\dots , v_N\rangle (\hat x_i)\label{eq:perturb} \\
        \frac{\partial}{\partial x_{1,1}}\psi_\delta(\mathbf{x}) &= \frac{\partial}{\partial x_{1,1}}\psi(\mathbf{x}) + \mathcal{J}(\mathbf{x})\left(\partial_1\ v_\delta(x_1)|v_2,\dots , v_N\rangle (\hat x_1) + \sum_{i=2}^N(-1)^{i+1} v_\delta(x_i)\frac{\partial}{\partial x_{1,1}}|v_2,\dots , v_N\rangle (\hat x_i)\right)\nonumber\\
        &+\frac{\partial}{\partial x_{1,1}}\mathcal{J}(\mathbf{x})\sum_{i=1}^N(-1)^{i+1} v_\delta(x_i)|v_2,\dots , v_N\rangle (\hat x_i)\label{eq:perturbgrad}\\
        \Delta_i\delta\psi(\mathbf{x})&=\Delta_i\psi(\mathbf{x})+ \mathcal{J}(\mathbf{x})\left(\sum_{j\neq i}(-1)^{j+1}v_\delta(x_j)\Delta_i|v_2,\dots, v_N\rangle (\hat x_j) + (-1)^{i+1} \Delta v_\delta(x_i)|v_2,\dots , v_N\rangle (\hat x_i)\right)\nonumber \\
        &+ \Delta_i\mathcal{J}(\mathbf{x})\left(\sum_{i=1}^N(-1)^{i+1} v_\delta(x_i)|v_2,\dots , v_N\rangle (\hat x_i)\right)\nonumber \\
        &+ 2 \nabla_i\mathcal{J}(\mathbf{x})\cdot \left((-1)^{i+1}\nabla v_\delta(x_i)|v_2,\dots , v_N\rangle (\hat x_i) + \sum_{j\neq i}(-1)^{j+1} v_\delta(x_j)\nabla_i|v_2,\dots , v_N\rangle (\hat x_j)\right)\label{eq:perturblap}.
    \end{align}
    Now, pick $\mathbf{x}=(x_1,\dots , x_N)$ according to Lemma \ref{lem:existenceofzeros} and, using these points, $v_\delta$ according to Lemma \ref{lem:interpolation} with $\nu^{(1)}, \nu^{(2)}$ to be determined. Then, with $\tau:=|v_2,\dots, v_N\rangle (\hat x_1)\neq 0$, $\mathcal{J}(\mathbf{x})\neq 0$ and noting that $\nabla v_\delta(x_1) = \frac{\nu^{(1)} x_1}{x_{1,1}}$ the formulas \eqref{eq:perturb}, \eqref{eq:perturbgrad}, \eqref{eq:perturblap} simplify to 
    \begin{align}
        \psi_\delta(\mathbf{x}) &= 0 \\
        \frac{\partial}{\partial x_{1,1}}\psi_\delta(\mathbf{x}) &= \frac{\partial}{\partial x_{1,1}}\psi(\mathbf{x}) + \nu^{(1)} \mathcal{J}(\mathbf{x})\tau \label{eq:gradperturb}\\
        \Delta\psi_\delta(\mathbf{x})&=\Delta\psi(\mathbf{x}) + \nu^{(2)} \mathcal{J}(\mathbf{x})\tau +\frac{2\nu^{(1)}}{x_{1,1}}\left(\nabla_1\mathcal{J}(\mathbf{x})\cdot x_1\right). \label{eq:lapperturb}
    \end{align}
    Since we can choose $\nu^{(1)}, \nu^{(2)}$ freely, it is easy to ensure that $ \frac{\partial}{\partial x_{1,1}}\psi_\delta(\mathbf{x})\neq 0$ and $\Delta\psi_\delta(\mathbf{x})\neq 0$. By scaling $v_\delta$  (and hence $(\nu^{(1)},\nu^{(2)})$), we can ensure via \eqref{eq:gradperturb}, \eqref{eq:lapperturb} that $\frac{\partial}{\partial x_{1,1}}\psi_{\delta,\lambda}(\mathbf{x})\neq 0$ and $\Delta\delta\psi_\lambda(\mathbf{x})\neq 0$, whenever $\lambda>0$ is sufficiently small. This proves \eqref{eq:regularvalueagain}.

    It remains to prove \eqref{eq:nonzerograd}. To this end, let $u\in V$ with $u(x_1)\neq 0$ and $u(x_i) = 0$ for all $i=2,\dots , N$ (such $u$ can be constructed by picking $v_r$ to vanish at $|x_3|,\dots , |x_N|$ but not on $|x_1|=|x_2|$, and $w\in \R^3$ with $w\cdot x_2 = 0$ and $w\cdot x_1\neq 0$. Then, $u(x):=(w\cdot x)v_r(|x|)$ satisfies the desired properties). Hence, it holds that
    $$
        |u,v_2,\dots , v_N\rangle(\mathbf{x}) = \sum_{i=1}^N (-1)^{1+i}u(x_i)|v_2,\dots , v_N\rangle (\hat x_i) = u(x_1)\tau\neq 0.
    $$
    Since $\mathcal{J}(\mathbf{x})\neq 0$ and $\tau\neq 0$, this proves the inequality in \eqref{eq:nonzerograd}. The equality in \eqref{eq:nonzerograd} is immediate from the multilinearity of the determinant. 
\end{proof}
Using Lemma \ref{lem:perturbation}, we are finally in a position to prove Theorem \ref{thm:mainmomentslater}.
\begin{proof}[Proof of Theorem \ref{thm:mainmomentslater}]
    Since the map $\theta\mapsto (\psi_\theta,\nabla_\theta\psi_\theta)$ is continuous with respect to the local topology, by Lemma \ref{lem:regularopen} the set 
    $$
        \mathcal{O}':=\{\theta\in \R^J\times [0,\infty)^{NBI}\times \R^{NB(L+1)^2KL}:\ \exists \mathbf{x}\in \R^{3\times N}:\  \psi_\theta(\mathbf{x})=0,\  \nabla \psi_\theta(\mathbf{x})\neq 0,\   \Delta\psi_\theta(\mathbf{x})\neq 0, \mbox{and}\  \nabla_\theta\psi_\theta(\mathbf{x})\neq 0\}
    $$
    is open in $\R^J\times [0,\infty)^{NBI}\times \R^{NB(L+1)^2KL}$.
    We will show that $\mathcal{O}'$ is also dense in $\R^J\times [0,\infty)^{NBI}\times \R^{NB(L+1)^2KL}$.
    To this end, let 
    $$
        \psi_\theta = \mathcal{J}_\beta \cdot \left(\sum_{b=1}^B |h_{(\bm{\alpha}_1^b,\mathbf{c}_1^b)},\dots ,h_{(\bm{\alpha}_N^b,\mathbf{c}_N^b)}\rangle\right).
    $$
    By removing a nowhere dense set of parameters we may assume that
    the functions $v_i:=h_{\bm{\alpha}_i^B,\mathbf{c}_i^B}$ are linearly independent in $C(\R^3)$ (to see this we note that for $(x_1,\dots, x_N)\in \R^{3\times N}$ the mapping $\Phi:(\bm{\alpha}_i^B,\mathbf{c}_i^B)_{i=1}^N\mapsto \det\left(v_i(x_j)\right)$ is real analytic and not identically zero, therefore the set of parameters
    $(\bm{\alpha}_i^B,\mathbf{c}_i^B)_{i=1}^N$ with $\Phi\left((\bm{\alpha}_i^B,\mathbf{c}_i^B)_{i=1}^N\right)=0$ is nowhere dense) and that $\bm{\alpha}_1^B=(\alpha_1,\dots , \alpha_I) \in [0,\infty)^I$ has pairwise distinct entries.
    
    Consider the finite dimensional vector space
    $$
        V:=\textrm{span}\{h_{\alpha_{1,i}^B,l,m,j}:\ l=0,\dots,L,\ m=-l,\dots , l,\ k=0,\dots, K-1,\ i=1,\dots , I\} 
    $$
    and observe that, since $\mathfrak{h}_1=\{\R^3 \ni x\mapsto w\cdot x:\ w\in \R^3\}$ and $L\geq 1$ it holds that
    $$
        V\supseteq \{\R^3\ni x\mapsto (w\cdot x)v_r(|x|):\ w\in \R^3\land v_r\in V_r\},
    $$
    where 
    $$
        V_r = \mathrm{span}\{r^k e^{-\alpha_{1,i}^B r}:\ i=1,\dots I, k=0,\dots , K-1\}.
    $$
    Since by assumption $KI\geq 3(N-1)$, it follows from Lemma \ref{lem:chebsys} that $V_r$ satisfies the $(N-1)$-HIP. 
    Therefore, using the linearity of $h_{\mathbf{\alpha}_1^B,\mathbf{c}_1^B}$ in the parameter $\mathbf{c}_1^B$, we can apply Lemma \ref{lem:perturbation} with 
    $\mathcal{J}=\mathcal{J}_\beta$ and $F=G_\beta\cdot\sum_{b=1}^{B-1} |h_{\mathbf{\alpha}_1^b,\mathbf{c}_1^b},\dots ,h_{\mathbf{\alpha}_N^b,\mathbf{c}_N^b}\rangle$, to conclude that an arbitrarily small perturbation $\theta'$ of $\theta$ produces $\mathbf{x}\in \R^{3\times N}$ with 
    \begin{equation}\label{eq:nicezeropsi}
        \psi_{\theta'}(\mathbf{x}) = 0,\ \nabla \psi_{\theta'}(\mathbf{x})\neq 0\quad \mbox{and}\quad \Delta \psi_{\theta'}(\mathbf{x})\neq 0.
    \end{equation}
    Moreover, by \eqref{eq:nonzerograd} there exists $u\in V$ with $\mathcal{J}_\beta(\mathbf{x})\cdot|u,v_2,\dots , v_N\rangle(\mathbf{x})\neq 0$. 
    By noting that 
    $$
        \nabla_{\mathbf{c}_1^B}\psi_{\theta'}(\mathbf{x})\neq 0\Leftrightarrow \exists u\in V: |u,v_2,\dots , v_N\rangle(\mathbf{x})\neq 0
    $$
    we conclude that also 
    \begin{equation}\label{eq:nicezerograd}
        \nabla_{\theta}\psi_{\theta'}(\mathbf{x})\neq 0
    \end{equation}
    holds true. 
    Together, \eqref{eq:nicezeropsi} and \eqref{eq:nicezerograd} establish  that $\theta'\in \mathcal{O}'$ and therefore $\mathcal{O}'$ is dense.  
    Finally, by Lemma \ref{lem:regularvalue} it holds that
    $$
        \mathcal{O}'\subset \left\{\theta\in \R^J\times [0,\infty)^{NBI}\times \R^{NB(L+1)^2KL}: \mathbb{E}\left[|E_{\theta}|^3\right]=\infty \quad \land \quad \mathbb{E}\left[|Z_\theta|^{3/2}\right]=\infty\right\}. 
    $$
    The theorem is proven.
\end{proof}
\section{Existence of Moments}\label{sec: Existence of Moments}
In this section, we establish the key fact that, whenever the functions $\psi_\theta$ are real-analytic, a guaranteed minimal number of moments exists for the random variables $E_\theta$ and $Z_\theta$. 
\subsection{Formulation of the Main Results}
Before stating and proving our main result, we begin by noting the following trivial moment properties. They will be needed in the analysis of the robust optimization methods introduced in Section \ref{sec:optimization}. 
\begin{theorem}\label{thm:momentexistenceenergy}
    Assume Setting \ref{set:Hamiltonian}. Then, for every $\psi\in H^2(\Omega)\setminus \{0\}$ it holds that
    \begin{equation}\label{eq:energyvar}
        \mathbb{E}\left[|E_\psi|^2\right]=\frac{\|\mathcal{H}\psi\|_{L^2}^2}{\|\psi\|^2_{L^2}}<\infty.
    \end{equation}
    For $\theta\mapsto \psi_\theta\in H^2(\Omega)\setminus\{0\}$ with $\theta\mapsto \psi_\theta$ differentiable in $L^2(\Omega)$ it holds that
    \begin{equation}\label{eq:logdervar}
        \mathbb{E}\left[|W_\theta|^2\right]=\frac{\|\nabla_\theta \psi_\theta\|_{L^2}^2}{\|\psi\|^2_{L^2}}<\infty.
    \end{equation}
    %
    %As a consequence, for every $p\le 2$ it holds that
    %%
    %\begin{equation}\label{eq:gradnocentering}
    %    \mathbb{E}\left[|Z_\theta|^p\right]<\infty \Leftrightarrow
    %    \mathbb{E}\left[\left| E_{L,\theta}(X) \nabla_\theta\log(|\psi_\theta(X)|) \right|^p\right]<\infty.
    %\end{equation}
    %%
\end{theorem}
\begin{proof}
    This follows directly from the definition.
\end{proof}
We now state the main result of this section.
\begin{theorem}\label{thm:momentexistencegrad}Assume Setting \ref{set:Hamiltonian}.
    Let $K\subset \Omega$ compact with smooth boundary, $\varepsilon >0$, $P\in\N$ and $\mathcal{P}\subset \R^P$ open so that the mapping $\mathcal{P}\times \Omega \ni (\theta,\mathbf{x})\mapsto \varphi_\theta(\mathbf{x})$ is in $C^\omega(\mathcal{P}\times \Omega)$ and $Q\subset \mathcal{P}$ compact. Assume further that $\mathcal{V}\in L^{5/4}_{\mathrm{loc}}$. Then, for any $\eta\in C^\infty(\Omega)$ with $\mbox{supp}(\eta)=K$ 
    %and $\forall\,  x\in K:\, \textup{dist}(x,\partial K)>\varepsilon\Rightarrow \eta(x) = 1$  such that 
     it holds for the parametrized family $\theta\mapsto \psi_\theta:=\eta \varphi_\theta$ that 
    \begin{equation}\label{eq:gradmomentthres}
        \sup\{p\in (0,\infty):\ \sup_{\theta\in Q}  \mathbb{E}[|Z_{\theta}|^p]<\infty \}\geq \frac{5}{4}.
    \end{equation}
\end{theorem}
\begin{remark}
    A function $\eta \in C^{\infty}(\Omega)$ satisfying the properties in the statement of Theorem \ref{thm:momentexistencegrad} can for example be constructed as follows. First, note that for $\varepsilon>0$ sufficiently small and for $\tau>0$ defining $A_{\tau}:=\{\mathbf{x}\in \R^d:\ \mathrm{dist}(\mathbf{x},\partial K)\le \tau\}\subset \R^d$  the distance function $\mathrm{dist}(\mathbf{x},\partial K)$ is in $C^\infty(A_{\varepsilon/2}\cap K)$. Now, let $s\in C^\infty(K)$ with 
    $s|_{K\setminus A_\varepsilon}\equiv 1$ and $s|_{K\cap A_{\varepsilon/2}} = \mathrm{dist}(\cdot , \partial K)$ and let $\eta(\mathbf{x}):=\exp(1)\cdot\exp\left(-\frac{1}{s(\mathbf{x})}\right)\cdot \1_K(\mathbf{x})$. It is a standard fact that $\eta \in C^\infty(\Omega)$ (this can be seen from the fact that the function $t\mapsto e^{-1/t}\cdot \1_{(0,\infty)}$ is in $C^\infty(\R)$). Furthermore, it holds that $\mbox{supp}(\eta)=K$ and $\forall x\in K:\ \mbox{dist}(x,\partial K)>\varepsilon\Rightarrow \eta(x) = 1$.
\end{remark}
We finally note that the energy gradient of real-analytic parametrizations are Lipschitz continuous. This property will be needed to prove convergence of the robust optimization methods introduced in Section \ref{sec:optimization}.
\begin{theorem}\label{thm:GradLip}Assume Setting \ref{set:Hamiltonian}
    and consider the mapping $\theta\mapsto \psi_\theta$ from Theorem \ref{thm:momentexistencegrad}. Assume additionally that \(\psi_\theta\neq 0\) for every \(\theta\in Q\).
    Then, the mapping $\theta\mapsto \psi_\theta$ is globally Lipschitz continuous as a map from $Q$ to $H^2(\Omega)$ and continuously differentiable as a map from $Q$ to $L^2(\Omega)$, with globally Lipschitz continuous derivative. Moreover, there exists a constant $C\in (0,\infty)$ such that
    \begin{equation}\label{eq:gradlipschitz}
        \forall \theta_1,\theta_2\in Q:\quad \left|\mathbb{E}[Z_{\theta_1}] -\mathbb{E}[Z_{\theta_2}]\right| \le C|\theta_1 - \theta_2|.
    \end{equation}
\end{theorem}
\begin{proof}
We first show (using the notation of Theorem \ref{thm:momentexistencegrad}) that $\theta\mapsto \psi_\theta$ is globally Lipschitz continuous as a map from $Q$ to $H^2(K)\subset H^2(\Omega)$. 
Let $f(\theta,\mathbf{x}):=\varphi_\theta(\mathbf{x})$. By assumption, this is a real-analytic function and consequently all its derivative, including $\nabla f, \ \nabla^2f, \nabla^3f$, are also real-analytic on $\mathcal{P}\times\Omega$ (and therefore bounded on compact sets). Since $Q\times K$ is compact this directly implies that there exists a constant $C\in (0,\infty)$ with
$$
    \|f(\theta_1,\cdot) - f(\theta_2,\cdot)\|_{L^\infty(K)} +
    \|\nabla f(\theta_1,\cdot) - \nabla f(\theta_2,\cdot)\|_{L^\infty(K)}
    +
    \|\nabla^2 f(\theta_1,\cdot) - \nabla^2 f(\theta_2,\cdot)\|_{L^\infty(K)}
    \le C|\theta_1 - \theta_2|\quad \forall \theta_1,\theta_2\in Q.
$$
Since $\eta\in C^\infty(K)$ this directly implies the existence of a (possibly different) constant $C\in (0,\infty)$ with 
$$
    \|\psi_{\theta_1}-\psi_{\theta_2}\|_{H^2(K)}\le C|\theta_1 - \theta_2|\quad \forall \theta_1,\theta_2\in Q.
$$
This shows that the map $\theta\mapsto \psi_\theta$ is globally Lipschitz as a map from $Q$ to $H^2(K)$.
Using similar arguments one can show that the mapping $\theta\mapsto \psi_\theta$ is infinitely often continuously differentiable as a map from $Q$ to $L^2(K)$, which in particular implies that
there exists $C\in (0,\infty)$ with
$$
    \|\psi_{\theta_1}- \psi_{\theta_2}\|_{L^2(K)} + \|\nabla_\theta\psi_{\theta_1}- \nabla_\theta\psi_{\theta_2}\|_{L^2(K)}\le C|\theta_1-\theta_2|\quad \forall \theta_1,\theta_2\in Q.
$$
It remains to show \eqref{eq:gradlipschitz}. To this end note that 
$$
    \mathbb{E}[Z_\theta]=\frac{1}{\|\psi_\theta\|_{L^2}^2}\int_K\Delta\psi_\theta(\mathbf{x})\cdot \nabla_\theta \psi_\theta(\mathbf{x})
    -\frac{1}{\|\psi_\theta\|_{L^2}^4}\langle \mathcal{H}\psi_\theta,\psi_\theta\rangle\cdot  \langle \nabla_\theta\psi_\theta,\psi_\theta\rangle.
$$
By the previous considerations (and using that $\mathcal{H}:L^2 \to H^2$ is bounded), this is Lipschitz continuous (in fact, it is easy to see that it is $C^\infty$) in $\theta$. Equation \eqref{eq:gradlipschitz} follows. 
\end{proof}
\subsection{Proof of the Main Result}
We now prove Theorem \ref{thm:momentexistencegrad}. We will make extensive use of the well-known structure on the local geometry of nodal sets of analytic functions, as exhibited by the Weierstrass preparation theorem.
\begin{theorem}[Weierstrass Preparation Theorem]\label{thm:weierstrassprep}
    Let $\Omega \subset \R^d$ open. Then, for each $f\in C^\omega(\Omega)$, and every $\mathbf{x}=(x_1,\dots , x_d)\in \Omega$ %with $\hat{\mathbf{x}}_{1}:=(x_2,\dots , x_d)$ 
    such that 
    $$
        \forall l=0,\dots , k-1:\ \frac{\partial^l}{\partial x_1^l}f(\mathbf{x})=0\quad \land \quad 
        \frac{\partial^k}{\partial x_1^k}f(\mathbf{x})\neq 0
    $$
    there is an interval $I$ with $x_1\in I$ and an open set $V\subset \R^{d-1}$ with $\hat{\mathbf{x}}_{1}\in V$ and $I\times V\subset \Omega$ and unique functions $a_i\in C^\omega(V)$, $i=0,\dots , k-1$, $b\in C^\omega(I\times V)$ with 
    \begin{multline}
        \forall l=0,\dots , k-1:\ a_l(\hat{\mathbf{x}}_{1}) = 0\ \land \ \\ \forall \mathbf{y}=(y_1,\hat{\mathbf{y}}_{1})\in I\times V:\ b(\mathbf{y})\neq 0\ \land \ f(\mathbf{y}) = \left((y_1-x_1)^k + \sum_{l=0}^{k-1}a_l(\hat{\mathbf{y}}_{1})(y_1-x_1)^l\right)\cdot b(\mathbf{y}).
    \end{multline}
\end{theorem}
\begin{proof}
    A proof for holomorphic functions can be found in \cite[Theorem 4.41]{korevaar2017several}, the real-analytic statement follows by noting that every real-analytic function can be locally extended to a holomorphic function. 
\end{proof}
With the Weierstrass preparation theorem at hand, we can now reduce the local integrability of the random variables $E_\theta, Z_\theta$ to studying integrability properties of the \emph{Weierstrass polynomials} $y_1 + \sum_{l=0}^{k-1}a_l(\mathbf{y}_{d-1})y_1^l$. This will involve quotients between gradients and Laplacians of the Weierstrass polynomial and the Weierstrass polynomial itself. Since applying the Laplacian requires us to differentiate the Weierstrass polynomial in every coordinate $y_i$ and since the root functions $a_l$ are in general not analytic, we need to make sure that a Weierstrass preparation is possible in every coordinate. This is ensured by the following Lemma. 
\begin{lemma}\label{lem:orderalign}
    Let $\Omega \subset \R^d$ open. Let $f\in C^\omega(\Omega)$, $\mathbf{x}=(x_1,\dots , x_d)\in \Omega$ 
    and 
    $$
        \mathrm{ord}_f(\mathbf{x}):=\min\{m\in \N_0:\ \exists \bm{\alpha}:=(\alpha_1,\dots , \alpha_d)\in \mathbb{N}_0^d:\ \sum_{i=1}^d\alpha_i = m\ \land \frac{\partial^{\alpha_1}}{\partial x_d^{\alpha_1}}\cdots \frac{\partial^{\alpha_d}}{\partial x_d^{\alpha_d}}f(\mathbf{x})\neq 0\}.
    $$
    Then, there exist orthonormal vectors $\mathbf{r}_1,\dots , \mathbf{r}_d\in \R^d$ such that 
    $$
        \forall i\in [d]:\ \forall l=0,\dots , \mathrm{ord}_f(\mathbf{x})-1:\ \frac{d^l}{dt^l}f(\mathbf{x}+t\mathbf{r}_i)|_{t=0}=0 \quad \land \quad \frac{d^{\mathrm{ord}_f(\mathbf{x})}}{dt^{\mathrm{ord}_f(\mathbf{x})}}f(\mathbf{x}+t\mathbf{r}_i)|_{t=0}\neq0 \;.
    $$
\end{lemma}
\begin{proof}
    Since for $l<\mathrm{ord}_f(\mathbf{x})$ all derivatives of $f$ at $\mathbf{x}$ of order $l$ vanish, it follows that for any $\mathbf{r}\in \R^d$ and $l<\mathrm{ord}_f(\mathbf{x})$ it holds that
    $$
        \frac{d^l}{dt^l}f(\mathbf{x}+t\mathbf{r})|_{t=0}=0.
    $$
    Furthermore, we note that with $H(\mathbf{r}):=\sum_{\bm{\alpha}=(\alpha_1,\dots , \alpha_d)\in \mathbb{N}_0^d:\ \sum_{i=1}^d\alpha_i = \mathrm{ord}_f(\mathbf{x})}\frac{\mathrm{ord}_f(\mathbf{x})!}{\bm{\alpha}!}\frac{\partial^{\alpha_1}}{\partial x_d^{\alpha_1}}\cdots \frac{\partial^{\alpha_d}}{\partial x_d^{\alpha_d}}f(\mathbf{x})\cdot \mathbf{r}^{\bm{\alpha}}$, we have that 
    \begin{equation}\label{eq:orderderivative}
        \frac{d^{\mathrm{ord}_f(\mathbf{x})}}{dt^{\mathrm{ord}_f(\mathbf{x})}}f(\mathbf{x}+t\mathbf{r})|_{t=0}=H(\mathbf{r}).
    \end{equation}
    Since this is a homogeneous polynomial, it follows that 
    $$
        B:=\{\mathbf{r}\in \mathbb{S}^{d-1}:\ H(\mathbf{r})=0\}
    $$
    is closed and has empty interior in $\mathbb{S}^{d-1}$.
    For $i\in [d]$, let $\pi_i:\left\{\begin{array}{ccc}O(d)& \to & \mathbb{S}^{d-1}\\ \mathbf{R}=(\mathbf{r}_1,\dots, \mathbf{r}_d)& \mapsto & \mathbf{r}_i\end{array}\right.$ and note that, since $\pi_i$ is continuous and open, the set $\pi_i^{-1}\left(B\right)$ is closed and has empty interior in $O(d)$. Therefore, there exists $\mathbf{R}=(\mathbf{r}_1,\dots , \mathbf{r}_d)\in  O(d)\setminus \bigcup_{i\in [d]}\pi_i^{-1}\left(B\right)$ which by 
    \eqref{eq:orderderivative} satisfies that $\frac{d^{\mathrm{ord}_f(\mathbf{x})}}{dt^{\mathrm{ord}_f(\mathbf{x})}}f(\mathbf{x}+t\mathbf{r}_i)|_{t=0}\neq0$ for all $i\in [d]$.
\end{proof}
Roughly speaking, the local integrability of $E_\theta, Z_\theta$ is determined by integrability properties of quantities related to Weierstrass polynomials. The following lemma establishes the necessary integrability properties for polynomials that will later be used in the proof of Theorem \ref{thm:momentexistencegrad}.
\begin{lemma}\label{lem:polynomint}
    Let $u(t) = t^k + \sum_{l=0}^{k-1}a_lt^l$, $t\in I$, and $I\subset \R$ a bounded interval. Then, for each $p<3/2$ there is a constant $C_1=C_1(I,\max_{l=0}^{k-1}|a_l|,p)$ with 
    \begin{equation}\label{eq:polynomint1}
        \int_I|u'(t)|^p|u(t)|^{2-2p}dt \le C_1
    \end{equation}
    and for each $p<5/4$ there is a constant $C_2=C_2(I,\max_{l=0}^{k-1}|a_l|,p)$ with 
    \begin{equation}\label{eq:polynomint2}
        \int_I|u''(t)|^p|u(t)|^{2-2p}dt \le C_2.
    \end{equation}
\end{lemma}
\begin{proof}
    Consider the factorization 
    $$
        u(t) = \prod_{l=1}^{k}(t-\alpha_l),\quad \alpha_l\in \C.
    $$
    Then, we have that 
    $$
        u'(t) = \sum_{i=1}^k\prod_{l\in [k]\setminus\{i\}}(t-\alpha_l),\quad u''(t)=\sum_{i\in [k],\ j\in [k]\setminus \{i\}}\prod_{l\in [k]\setminus\{i,j\}}(t-\alpha_l)
    $$
    and therefore
    $$
        \frac{u'(t)}{u(t)} = \sum_{i=1}^k(t-\alpha_i)^{-1},\quad \frac{u''(t)}{u(t)}=\sum_{i\in [k],\ j\in [k]\setminus \{i\}}(t-\alpha_i)^{-1}(t-\alpha_j)^{-1}.
    $$
    By the Cauchy root bound (together with the fact that the highest degree term in $p$ has coefficient $1$), it holds that
    $$
        |\alpha_l|\le 1 + \max_{l=0}^{k-1}|a_l|.
    $$
    Hence, we obtain that
    $$
        \left|\frac{u'(t)}{u(t)}\right|^p|u(t)|^{2-p}\le k^{p-1}\sum_{i=1}^k |t-\alpha_i|^{2-2p}\prod_{l\in [k]\setminus\{i\}}|t-\alpha_l|^{2-p}\le
        k^{p-1}(|t|+1+\max_{l=0}^{k-1}|a_l|)^{(2-p)(k-1)}\sum_{i=1}^k |t-\alpha_i|^{2-2p},
    $$
    which is integrable on $I$ if $2-2p>-1$ or equivalently $p<\frac{3}{2}$. This proves \eqref{eq:polynomint1}.

    To establish \eqref{eq:polynomint2} we note that with some constant $C_{k,p}$
\[
\begin{aligned}
    \left|\frac{u''(t)}{u(t)}\right|^p|u(t)|^{2-p}
    &\le
    C_{k,p}
    \sum_{i\in [k]}\sum_{j\in [k]\setminus\{i\}}
    |t-\alpha_i|^{2-2p}|t-\alpha_j|^{2-2p}
    \prod_{l\in [k]\setminus\{i,j\}}|t-\alpha_l|^{2-p}.
\end{aligned}
\]
Hence, it is enough to bound terms of the form
\[
    \int_I |t-\alpha_i|^{2-2p}|t-\alpha_j|^{2-2p}dt.
\]
By Cauchy-Schwarz,
\[
\begin{aligned}
    \int_I |t-\alpha_i|^{2-2p}|t-\alpha_j|^{2-2p}\,dt
    &\le
    \left(\int_I |t-\alpha_i|^{4-4p}\,dt\right)^{1/2}
    \left(\int_I |t-\alpha_j|^{4-4p}\,dt\right)^{1/2},
\end{aligned}
\]
which is integrable on I if 
\(4-4p>-1\), or equivalently \(p<\frac54\). 
This, together with the estimate for $|\alpha_l|$ above, proves \eqref{eq:polynomint2}.

\end{proof}
We are now ready to prove Theorem \ref{thm:momentexistencegrad}.
\begin{proof}[Proof of Theorem \ref{thm:momentexistencegrad}]

    In view of the definition of $Z_\theta$ we need to show that for $p<\frac54$ it holds that
    $$
        \sup_{\mu\in Q} \int_{\Omega}\left|\frac{\Delta \psi_\mu(\mathbf{y})\cdot\nabla_\mu \psi_\mu(\mathbf{y})}{\psi_\mu(\mathbf{y})^2}\right|^p|\psi_\mu(\mathbf{y})|^2d\mathbf{y}<\infty.
    $$
    Denote $f:(\theta,\mathbf{x})\mapsto \varphi_\theta(\mathbf{x})$, and for $(\theta,\mathbf{x})\in Q\times K$ fixed, let $k=\mathrm{ord}_{\varphi_\theta}(\mathbf{x})$.
    By Lemma \ref{lem:orderalign}, there are orthonormal vectors $\mathbf{r}_1,\dots , \mathbf{r}_d\in \R^d$ with 
    $$
        \forall i\in [d]:\ \forall l=0,\dots , k-1:\ \frac{d^l}{dt^l}f(\theta,\mathbf{x}+t\mathbf{r}_i)|_{t=0}=0 \quad \land \quad \frac{d^{k}}{dt^{k}}f(\theta,\mathbf{x}+t\mathbf{r}_i)|_{t=0}\neq 0 \;.
    $$
    For ease of notation from now on we assume that $\mathbf{r}_1,\dots , \mathbf{r}_d$ is the standard Euclidean basis and consequently that
    \begin{equation}\label{eq:prepprep}
        \forall i\in [d]:\ \forall l=0,\dots , k-1:\ \frac{\partial^l}{\partial x_i^l}f(\theta,\mathbf{x})=0 \quad \land \quad \frac{\partial^{k}}{\partial x_i^{k}}f(\theta,\mathbf{x})\neq 0.
    \end{equation}
    Important Remark: Since ultimately we will develop estimates for $\Delta \varphi_\theta$ and since $\Delta$ is invariant under orthogonal coordinate transforms this reduction will be harmless. 
    
    Let us introduce some convenient notation: For $i\in [d]$ and $\mathbf{y}\in \R^d$ we denote $\hat{\mathbf{y}}_i:=(y_1,\dots , y_{i-1},y_{i+1},\dots , y_d)\in \R^{d-1}$ and for $I\subset \R$ and a product set $V=\bigtimes_{j=1}^{d-1}V_j \subset \R^{d-1}$ denote $I\times_i V:= \bigtimes_{j=1}^{i-1}V_j\times I\times \bigtimes_{j=i+1}^{d-1}V_j$. Due to \eqref{eq:prepprep}, we can now apply Theorem \ref{thm:weierstrassprep} simultaneously in each coordinate of $\mathbf{x}$ (and, for ease of notation, assuming that $x_1=0$), yielding for each $i\in [d]$ an interval $I_i$ with $x_i\in I_i$, an open product set $V_i\subset \R^{d-1}$ with $\hat{\mathbf{x}}_{i}\in V_i$ and $I_i\times_i V_i\subset \Omega$, an open set $W_i\subset \mathcal{P}$ with $\theta\in W_i$, as well as unique functions $a_l^i\in C^\omega(W_i\times V_i)$, $l=0,\dots , k-1$, $b^i\in C^\omega(W_i\times I_i\times_i V_i)$ with 
    \begin{align}
        \forall l=0,\dots , k-1:\ a_l^i(\theta,\hat{\mathbf{x}}_{i}) = 0\label{eq:prep1}\\
        \forall (\mu,\mathbf{y})\in W_i\times I_i\times_i V_i:\ b^i(\mu,\mathbf{y})\neq 0\label{eq:prep2}\\
        \forall (\mu,\mathbf{y})\in W_i\times I_i\times_i V_i:\ f(\mu,\mathbf{y}) = \left(y_i^k + \sum_{l=0}^{k-1}a_l^i(\mu,\hat{\mathbf{y}_{i}})y_i^l\right)\cdot b^i(\mu,\mathbf{y}).
        \label{eq:prep3}
    \end{align}
    By shrinking the sets $W_i,I_i,V_i$ if necessary, we can in particular assume that $\ |b^i(\mu,\mathbf{y})|>\epsilon> 0$ for all $(\mu,\mathbf{y})\in W_i\times I_i\times_i V_i$ and $i\in[d]$. We now aim to show that 
    \begin{multline}\label{eq:integrability}
        \sup_{\mu\in \bigcap_{i=1}^dW_i} \int_{\bigcap_{i=1}^dI_i\times_iV_i}\left|\frac{\Delta f(\mu,\mathbf{y})\cdot \nabla_\mu f(\mu,\mathbf{y})}{f(\mu,\mathbf{y})^2}\right|^p|f(\mu,\mathbf{y})|^2d\mathbf{y}<\infty \ \land 
\\        \sup_{\mu\in \bigcap_{i=1}^dW_i} \int_{\bigcap_{i=1}^dI_i\times_iV_i}\left|\frac{|\nabla_{\mathbf{y}} f(\mu,\mathbf{y})|\cdot |\nabla_\mu f(\mu,\mathbf{y})|}{f(\mu,\mathbf{y})^2}\right|^p|f(\mu,\mathbf{y})|^2d\mathbf{y}<\infty.
    \end{multline}
    Since due to analyticity $\nabla_\mu f(\mu,\mathbf{y})$ is locally bounded, \eqref{eq:integrability} follows if we can establish that
    \begin{equation}\label{eq:integrability2}
        \forall i\in[d],\ l\in \{1,2\}:\ \sup_{\mu\in W_i} \int_{I_i\times_iV_i}\left|\frac{\frac{\partial^l}{\partial y_i^l} f(\mu,\mathbf{y})}{f(\mu,\mathbf{y})^2}\right|^p|f(\mu,\mathbf{y})|^2d\mathbf{y}<\infty.
    \end{equation}
    In order to prove \eqref{eq:integrability2}, we use the representations \eqref{eq:prep1}, \eqref{eq:prep2}, and \eqref{eq:prep3}. Setting $u_{i,\mu,\hat{\mathbf{y}}_i}(y_i):=y_i^k + \sum_{l=0}^{k-1}a_l^i(\mu,\hat{\mathbf{y}_{i}})y_i^l$, it follows from \eqref{eq:prep2} that \eqref{eq:integrability2} holds true if
    for all $i\in [d]$:
    \begin{equation}\label{eq:integrability3}
         \sup_{\mu\in W_i} \int_{I_i\times_iV_i}\left|\frac{u_{i,\mu,\hat{\mathbf{y}}_i}'(y_i)}{u_{i,\mu,\hat{\mathbf{y}}_i}(y_i)^2}\right|^p|u_{i,\mu,\hat{\mathbf{y}}_i}(y_i)|^2d\mathbf{y}<\infty\ \land \ \sup_{\mu\in W_i} \int_{I_i\times_iV_i}\left|\frac{u_{i,\mu,\hat{\mathbf{y}}_i}''(y_i)}{u_{i,\mu,\hat{\mathbf{y}}_i}(y_i)^2}\right|^p|u_{i,\mu,\hat{\mathbf{y}}_i}(y_i)|^2d\mathbf{y}<\infty\
    \end{equation}
    Finally, we note that, if $p<\frac54$, \eqref{eq:integrability3} is now a direct consequence of Lemma \ref{lem:polynomint} using the fact that the coefficient functions $a_l^i(\mu,\hat{\mathbf{y}}_i)$ are bounded on $I_i\times_iV_i$. We have thus established \eqref{eq:integrability}.

    Next, we note that 
    \begin{multline}\label{eq:integrability4}
    \sup_{\mu\in \bigcap_{i=1}^dW_i} \int_{\bigcap_{i=1}^dI_i\times_iV_i}\left|\frac{\Delta \psi_\mu(\mathbf{y})\cdot\nabla_\mu \psi_\mu(\mathbf{y})}{\psi_\mu(\mathbf{y})^2}\right|^p|\psi_\mu(\mathbf{y})|^2d\mathbf{y} =\\
        \sup_{\mu\in \bigcap_{i=1}^dW_i} \int_{\bigcap_{i=1}^dI_i\times_iV_i}\left|\frac{\Delta (\eta(\mathbf{y})f(\mu,\mathbf{y}))\cdot \eta(\mathbf{y})\nabla_\mu f(\mu,\mathbf{y})}{\eta(\mathbf{y})^2f(\mu,\mathbf{y})^2}\right|^p|\eta(\mathbf{y})f(\mu,\mathbf{y})|^2d\mathbf{y}<\infty
    \end{multline}
    is a direct consequence of \eqref{eq:integrability} and properties of $\eta$:
    Define
    $D:=\bigcap_{i=1}^d(I_i\times_iV_i)$ and
    $W:=\bigcap_{i=1}^d W_i$.
    Shrinking the sets \(I_i,V_i,W_i\), if necessary, we may assume that
    \(f,\nabla_\mu f,\eta,\nabla\eta,\Delta\eta\) are uniformly
    bounded on \(W\times D\).

Since \(\psi_\mu=\eta f(\mu,\cdot)\) and
\(\nabla_\mu\psi_\mu=\eta\nabla_\mu f\), the integrand in
\eqref{eq:integrability4} is
equal to
\[
    |\Delta(\eta f)|^p\,|\nabla_\mu f|^p\,
    |\eta|^{2-p}|f|^{2-2p}.
\]
By the product rule,
\[
    \Delta(\eta f)=\eta\,\Delta f
    +2\nabla\eta\cdot\nabla f
    +(\Delta\eta)f .
\]
Hence, for \(p<5/4\), there exists a constant \(C<\infty\), independent of
\(\mu\in W\), such that on \(D\)
\begin{align*}
    &|\Delta(\eta f)|^p|\nabla_\mu f|^p|\eta|^{2-p}|f|^{2-2p}  \\
    &\qquad\le C\left(
        |\Delta f|^p|\nabla_\mu f|^p|f|^{2-2p}
        +
        |\nabla f|^p|\nabla_\mu f|^p|f|^{2-2p}
        +
        |f|^{2-p}
    \right).
\end{align*}
The first term is controlled by the \(l=2\) case of
\eqref{eq:integrability2}, since
\(\Delta f=\sum_{j=1}^d\partial_j^2 f\) and \(\nabla_\mu f\) is uniformly
bounded. The second term is controlled by the \(l=1\) case of
\eqref{eq:integrability2}, again using the uniform boundedness of
\(\nabla_\mu f\). Finally, the third term is uniformly integrable on \(D\),
because \(f\) is bounded on \(W\times D\), \(D\) has finite measure, and
\(2-p>0\). Therefore
\[
    \sup_{\mu\in W}
    \int_D
    \left|
        \frac{\Delta\psi_\mu(\mathbf y)\cdot\nabla_\mu\psi_\mu(\mathbf y)}
             {\psi_\mu(\mathbf y)^2}
    \right|^p
    |\psi_\mu(\mathbf y)|^2\,d\mathbf y
    <\infty,
\]
which proves \eqref{eq:integrability4}.
    
    In summary, \eqref{eq:integrability4} yields an open covering of $Q\times K$ by sets $W\times V$ with 
    $$
        \sup_{\mu\in W} \int_{V}\left|\frac{\Delta \psi_\mu(\mathbf{y})\cdot\nabla_\mu \psi_\mu(\mathbf{y})}{\psi_\mu(\mathbf{y})^2}\right|^p|\psi_\mu(\mathbf{y})|^2d\mathbf{y}<\infty
    $$
    By compactness, there is a global constant with 
    $$
        \sup_{\mu\in Q} \int_{\Omega}\left|\frac{\Delta \psi_\mu(\mathbf{y})\cdot\nabla_\mu \psi_\mu(\mathbf{y})}{\psi_\mu(\mathbf{y})^2}\right|^p|\psi_\mu(\mathbf{y})|^2d\mathbf{y}\le C<\infty.
    $$
    Since $\mathcal{H}=-\Delta + \mathcal{V}$, this, the compactness of $\Omega$, the analyticity of $f$ and the assumption that $\mathcal{V}\in L^{5/4}_{\mathrm{loc}}$ imply that there is a global constant with
    $$
        \sup_{\mu\in Q} \int_{\Omega}\left|\frac{\mathcal{H} \psi_\mu(\mathbf{y})\cdot\nabla_\mu \psi_\mu(\mathbf{y})}{\psi_\mu(\mathbf{y})^2}\right|^p|\psi_\mu(\mathbf{y})|^2d\mathbf{y}\le C<\infty,
    $$
    which yields \eqref{eq:gradmomentthres}.
    
\end{proof}

\section{PS-Clip-VMC: Provably Convergent Stochastic Optimization}\label{sec:optimization}

In this section, we propose a robust gradient estimator based on per-sample clipping of both local energies and gradients. We show that, under weak moment assumptions, the resulting optimization method converges to a stationary point both in expectation and with high probability. The estimator and some useful notation are introduced in the following setting.

\begin{setting}[PS-Clip-VMC]\label{set:optimization}
    Assume Setting \ref{set:Hamiltonian} and consider a parametrized model $\mathbb{R}^P\ni \theta \mapsto \psi_\theta\in H^2(\Omega)$. For $\theta\in \mathbb{R}^P$ define the \emph{loss} $L(\theta):= \frac{\langle \mathcal{H}\psi_\theta,\psi_\theta\rangle}{\|\psi_\theta\|_{L^2}^2}=\mathbb{E}[E_\theta]$.\\
    Let $L^\ast:=\inf_{\theta\in \mathbb{R}^P}L(\theta)$ and assume that $L^\ast >-\infty$ (this holds if the bilinear form associated with $\mathcal{H}$ is bounded below; see for example \cite[Theorem 4.3]{kato:76:perturbation} for suitable conditions for this to hold).\\
    For $n\in \mathbb{N}$, $\alpha, \beta \in (0,\infty)$, $\theta\in \mathbb{R}^P$, and i.i.d samples $X_1, \ldots X_n \sim X_\theta$, set $E_i = E_{L,\theta}(X_i)$, $W_i = W_{\theta}(X_i)$ and let the \emph{clipped gradient estimator} be defined as
    \begin{equation}\label{eq:clipped_grad_est}
        \begin{aligned}
        G_n(\theta,\alpha,\beta )&=\frac{2}{n} \sum_{k=1}^n \gamma_{\alpha,p,k}\left(E_kW_k\right)-\frac{2}{n^2-n} \sum_{k \neq j} \gamma_{\beta,2,k}\left(E_k\right) \gamma_{\beta,2,k}\left(W_j\right) \;,
        \end{aligned}
    \end{equation}
    where $\gamma_{\alpha,p,k}(\cdot)$ is the clipping function given by
    \begin{equation}\label{eq:clip_factor}
        \gamma_{\alpha,p,k}(v) \coloneqq v\cdot \min \left\{1, \frac{\alpha k^{\frac{1}{p}}}{\left|v\right|} \right\} \quad \text{for any } v \in \R^d.
    \end{equation}
    For $M,n\in \mathbb{N}$, $\theta_1\in \mathbb{R}^P$, $\alpha, \beta\in \mathbb{R}$, $\eta_1,\dots , \eta_{M-1}\in (0,\infty)$ define the \emph{PS-Clip-VMC parameter update rule}
    $$
        \theta_{m+1}:=\theta_m - \eta_m G_n(\theta,\alpha,\beta ),\quad m=1,\dots , M-1,
    $$
    where for each parameter update $\theta_m\mapsto \theta_{m+1}$, the clipped gradient estimator $G_n(\theta,\alpha,\beta )$ is evaluated on i.i.d. samples $E_1,\dots , E_n \sim E_{\theta_m}$ and $W_1,\dots , W_n \sim W_{\theta_m}$.
    Finally, we let $\Delta_1:= L(\theta_1) - L^\ast>0$ the \emph{initialization gap}.
\end{setting}
We now state the assumptions under which the convergence results in this section will be proved.

\begin{assumptions}\label{set: SGD assumptions}
Assume Settings \ref{set:optimization} and \ref{set:Hamiltonian}, and that there are $p\in(1,2]$ and $u, \sigma, C \in (0,\infty)$ such that for any $\theta \in \mathbb{R}^P$,
\begin{itemize}
\item[(i)] $\E \left[|E_\theta|^2\right]\leq \sigma^2$
\item[(ii)] $\E\left[\left|W_\theta\right|^2\right]\leq \sigma^2$
\item[(iii)] $\E\left[\left|E_\theta W_\theta\right|^p \right]\leq u^p$
\item[(iv)] $|\nabla_\theta L(\theta_1)-\nabla_\theta L(\theta_2)|\leq C|\theta_1-\theta_2|$ for any $\theta_1, \theta_2\in \mathbb{R}^P$.
\end{itemize}
\end{assumptions}
\begin{remark}
    Instead of Item (iv) in Assumptions \ref{set: SGD assumptions}, one could also assume
\begin{itemize}
\item[(v)] $\E\left[\left|\mathrm{H}_\theta \psi_\theta / \psi_\theta\right|^2\right] \leq \sigma^2$
\item[(vi)] $\E\left[\left|\nabla_\theta H \psi_\theta / \psi_\theta\right|^2\right] \leq \sigma^2$,
\end{itemize}
which together with (i)-(iii) imply (iv); see \cite{abrahamsen2024convergence}. In (v), $\mathrm{H}_\theta\psi_\theta$ is the Hessian of $\psi_\theta$.
%We also assume that $L$ is bounded from below by some $L^* \in \mathbb{R}$, and we denote the initialization gap by $\Delta_1 := L(\theta_1) - L^*$.
\end{remark}
\begin{remark}\label{rem: energy clipping}
    It is common practice in variational Monte Carlo methods for solving the Schr\"odinger equation to clip the local energies during training \cite{hermann2023abinitioReview,hermann2020deep,gao2022pesnet, gao2023samplingFree, vonGlehn2023psiformer, gerard2022goldStandard}. Specifically, for some fixed multiplier $\alpha \in (0,\infty)$ and samples $E_1,\dots , E_n\sim E_{\theta_t}$, $W_1,\dots , W_n\sim W_{\theta_t}$, the gradient estimator used for the parameter update $\theta_t \mapsto \theta_{t+1}$ is given by
\begin{equation*}
G(\theta_t,\alpha) = \frac{1}{n}\sum_{k=1}^n \left(\clip_{[\mu-\alpha\sigma,\mu+\alpha\sigma]}\left(E_k\right)-\frac{1}{n}\sum_{i=1}^n \clip_{[\mu-\alpha\sigma,\mu+\alpha\sigma]}\left(E_i\right)\right)W_k\;,
\end{equation*}
where $\mu = \frac{1}{n}\sum_{i=1}^nE_i$ and $\sigma = \frac{1}{n}\sum_{i=1}^n |E_i-\mu|$. Our analysis suggests that the empirical success of this clipping strategy is rooted in the heavy-tailed nature of the underlying optimization problem. Moreover, our results show that local-energy clipping is sufficient to guarantee convergence in expectation, but not with high probability; see Appendix \ref{sec: In expectation convergence of Energy Clipping}. PS-Clip-VMC addresses this issue by additionally clipping the per-sample logarithmic gradients of the wave function. In Section \ref{sec: Experiments}, we provide empirical evidence that this method is significantly more robust than the common practice of only clipping the local energy.
\end{remark}

\subsection{Convergence in Expectation}\label{sec: Convergence in Expectation}
In this section we prove that under Assumptions \ref{set: SGD assumptions} PS-Clip-VMC converges in expectation to a stationary point. The result is formulated in the following theorem.
\begin{theorem}\label{thm: convergence in exp of vartiational monte carlo}
      Assume Settings \ref{set:optimization} and Assumptions \ref{set: SGD assumptions}. Let $\alpha = u$ and $\beta\in (0,\infty]$. Then, there is a constant $K_{u,\sigma,\beta}>0$ that only depends on $u$, $\sigma$ and $\beta$, such that for any $\eta_m <1/C$ the iterates generated by PS-Clip-VMC satisfy
\begin{align*}
        \sum_{m=1}^M\frac{\eta_m\E(|\nabla_\theta L(\theta_m)|^2)}{\sum_{m=1}^M\eta_m} \leq \frac{2\Delta_1}{\sum_{m=1}^M\eta_m} + \frac{K_{u,\sigma,\beta}}{n^{\frac{2(p-1)}{p}}}\;.
    \end{align*}
    In particular, if $\eta_m = 1/(2C)$ is constant and $n \geq M^{\frac{p}{2(p-1)}}$ we obtain that 
    \begin{align*}
        \frac{1}{M}\sum_{m=1}^M\E(|\nabla_\theta L(\theta_m)|^2)\leq \frac{\Delta_1 C +K_{u,\sigma,\beta}}{M}\;.
    \end{align*}
\end{theorem}

To prove Theorem \ref{thm: convergence in exp of vartiational monte carlo} we begin with the following proposition which provides a standard estimate of the average squared gradient norm of the iterates generated by variational Monte Carlo SGD.
\begin{proposition}\label{prop: gradient bound for iterates}
Assume Settings \ref{set:optimization} and Assumptions \ref{set: SGD assumptions}(iv). Then, for any $\eta_m< 1/C$, the iterates of variational Monte-Carlo SGD satisfy
\begin{align}
            \sum_{m=1}^M\eta_m|\nabla_\theta L(\theta_m)|^2 &\leq 2\Delta_1+\sum_{m=1}^M\eta_m|\nabla_\theta L(\theta_m)-G_n(\theta_m,\alpha,\beta)|^2.
        \end{align}
\end{proposition}
\begin{proof}
For simplicity we write $G_n(\theta_m)$ for $G_n(\theta_m,\alpha,\beta)$. Then, for every $m = 1, \ldots, M-1$ we have that
    \begin{align*}
        L(\theta_{m+1})-L(\theta_m) &\leq -\eta_m\langle\nabla_\theta L(\theta_m),G_n(\theta_m)\rangle+\frac{C\eta_m^2}{2}|G_n(\theta_m)|^2\\
        &=-\frac{\eta_m}{2}\left(|\nabla_\theta L(\theta_m)|^2+|G_n(\theta_m)|^2-|\nabla_\theta L(\theta_m)-G_n(\theta_m)|^2\right)+\frac{C\eta_m^2}{2}|G_n(\theta_m)|^2\\
        &\leq -\frac{\eta_m}{2}\left(|\nabla_\theta L(\theta_m)|^2-|\nabla_\theta L(\theta_m)-G_n(\theta_m)|^2\right) \;,
    \end{align*}
    where the equality in the second line is simply the polarization identity for inner products. Summing over $m$ and telescoping we obtain that
    \begin{align*}
            \sum_{m=1}^M\eta_m|\nabla_\theta L(\theta_m)|^2 &\leq 2\Delta_1+\sum_{m=1}^M\eta_m|\nabla_\theta L(\theta_m)-G_n(\theta_m)|^2.
        \end{align*}
   \end{proof}
By Proposition \ref{prop: gradient bound for iterates}, we see that, to prove convergence of the algorithm, we need to find an upper bound for the expected squared error $\E\!\left[\left|\nabla_\theta L(\theta_m) - G_n(\theta_m,\alpha,\beta)\right|^2\right]$. Note that, since we do not assume $|\nabla_\theta L(\theta)|$ to have bounded variance, this would not be possible if, instead of $G_n$, one were to use a non-robust gradient estimator such as the empirical mean.

We begin with the following lemma, which provides a generic in-expectation error bound for random variables with a finite $p$-th moment. In particular, when applied to the random variables $Y_k = E_kW_k$, it yields an estimate of the error generated by the first term in (\ref{eq:clipped_grad_est}).
\begin{lemma}\label{lemm: generic convergence in espectation}
        Let $Y_1, \ldots, Y_n$ be i.i.d. random variables such that $\E\left|Y_1\right|^{p}\leq u^p<\infty$ for some $p\in (1,2]$. Then, for any $\alpha\geq u$ we have that
    \begin{equation*}
    \E\left|\frac{1}{n}\sum_{k= 1}^n\gamma_{\alpha,p,k}(Y_k)-\E(Y_1)\right|^2 \leq 8\alpha^2n^{-\frac{2(p-1)}{p}}\;,
    \end{equation*}
    where $\gamma_{\alpha,p,k}(\cdot)$ is the clipping function defined in (\ref{eq:clip_factor}).
\end{lemma}
\begin{proof}
    For simplicity, we set $\bary_k = \gamma_{\alpha,p,k}(Y_k)$ in the following.
    We set, $\mu = \E(Y_1)$ and $\mu_k = \E(\bary_k)$. Then,
    \begin{equation}\label{eq: expectation bias}
    \begin{aligned}
        |\mu_k-\mu| &= \left|\E\left(\left(\frac{\alpha k^{\frac{1}{p}}}{|Y_k|}-1\right)Y_k\1_{(|Y_k|>\alpha k^{1/p})}\right)\right|\leq \E\left(Y_k1_{(|Y_k|>\alpha k^{1/p})}\right)\\
        &\leq \E\left(|Y_k|^{p}\1_{(|Y_k|>\alpha k^{1/p})}\right)\alpha^{1-p} k^{\frac{1-p}{p}}\leq u^p\alpha^{1-p} k^{\frac{1-p}{p}} \;.
    \end{aligned}
    \end{equation}
    Hence, we obtain that
    \begin{equation}\label{eq: bias term}
    \begin{aligned}
        \E\left|\frac{1}{n}\sum_{k= 1}^n\left(\mu_k-\mu\right)\right|^2 &\leq \frac{1}{n^2}\left(\sum_{k= 1}^nu^p\alpha^{1-p} k^{\frac{1-p}{p}}\right)^2
         \leq \frac{u^{2p}\alpha^{2(1-p)}}{n^2}\left(2n^{\frac{1}{p}}\right)^2\\
         &= 4u^{2p}\alpha^{2(1-p)}n^{-\frac{2(p-1)}{p}} \leq 4\alpha^2n^{-\frac{2(p-1)}{p}}\;.
    \end{aligned}
    \end{equation}
    Further, we have that
    \begin{equation}\label{eq: lln term}
    \begin{aligned}
        \E\left|\frac{1}{n}\sum_{k= 1}^n\left(\bary_k-\mu_k\right)\right|^2 &=\frac{1}{n^2}\sum_{k= 1}^n\E\left|\bary_k-\mu_k\right|^2 \leq \frac{1}{n^2}\sum_{k= 1}^n\E\left|\bary_k\right|^2\\
        &\leq \frac{1}{n^2}\sum_{k= 1}^n\E\left|Y_k\right|^{p}\alpha^{2-p} k^\frac{2-p}{p} \leq \frac{4u^p}{n^2}\alpha^{{2-p}}n^\frac{2}{p}\\
        &=4u^{p}\alpha^{2-p}n^{-\frac{2(p-1)}{p}}\leq 4\alpha^2n^{-\frac{2(p-1)}{p}}\;,
    \end{aligned}
    \end{equation}
    where in the second inequality, we used the fact that by definition of $\gamma_k$ we have that $\left|\bary_k\right|^2 = \left|\bary_k\right|^p\left|\bary_k\right|^{2-p}\leq \left|\bary_k\right|^p\alpha^{2-p} k^\frac{2-p}{p}$.
    Combining (\ref{eq: bias term}) and (\ref{eq: lln term}), we obtain that
    \begin{align*}
        \E\left|\frac{1}{n}\sum_{k= 1}^n\bary_k-\mu\right|^2 &\leq \E\left|\frac{1}{n}\sum_{k= 1}^n\left(\bary_k-\mu_k\right)\right|^2 + \E\left|\frac{1}{n}\sum_{k= 1}^n\left(\mu_k-\mu\right)\right|^2\leq 8\alpha^2n^{-\frac{2(p-1)}{p}} \;,
    \end{align*}
    which concludes the proof.
    \end{proof}
We now prove a similar error bound for the second term in (\ref{eq:clipped_grad_est}). The proof follows a similar idea to the previous result. However, it requires some additional steps, since, due to the double sum, the random variables involved are not necessarily independent.
\begin{lemma}\label{lem: L2 conv of covariance estimator}
Assume Settings \ref{set:optimization} and Assumptions \ref{set: SGD assumptions}(i-iii). Then, for any $\beta>0$,
    \begin{align*}
        \E\left[\left|\frac{1}{n(n-1)} \sum_{k \neq j} \gamma_{\beta,2,k}(E_k)\gamma_{\beta,2,k}(W_j)- \E(E_1)\E(W_1)\right|^2\right] \leq \left(2+\frac{3}{n-1}+\left(1+\frac{2\sigma}{\beta\sqrt{n}}\right)^2\frac{16\sigma^2}{\beta^2}\right)\frac{\sigma^4}{n}\;.
    \end{align*}
    \end{lemma}
\begin{proof}
    We set $\barz_k = E_k\gamma_{\beta,2,k}(E_k)$, $\barw_k = W_k\gamma_{\beta,2,k}(W_k)$ and $\mu_{k,j} = \E(\barz_k\barw_j)$. First, we note that since $E_k$ and $W_j$ are independent for any $k\neq j$ we have that
    \begin{equation}\label{eq: basic ineq}
    \begin{aligned}
    \E\left[\left|\sum_{k\neq j} \barz_k\barw_j -\E(E_1)\E(W_1)\right|^2\right] &\leq \E\!\left[\left| \sum_{k \ne j} \barz_k \barw_j - \mu_{k,j} \right|^2
\right]+ \left|\sum_{k\neq j}\E[E_kW_j]-\mu_{k,j}\right|^2\;.
    % &\leq\left(\frac{3}{n-1}+2\right)\frac{\sigma^4}{n}+\left(u+\frac{\sigma^2}{2\sqrt{\beta n}}\right)^2\frac{2\sigma^4}{n^3\beta}\\
    % &=\left(2+\frac{3}{n-1}+\frac{2}{n^2\beta}\left(u+\frac{\sigma^2}{2\sqrt{\beta n}}\right)^2\right)\frac{\sigma^4}{n} \;.
\end{aligned}
\end{equation}
We begin by bounding the first term on the right-hand side of (\ref{eq: basic ineq}).
\begin{equation}\label{eq: first term rhs}
\begin{aligned}
 \E\!\left[ 
\left| \sum_{k \ne j} \barz_k \barw_j - \mu_{kj} \right|^2
\right]
&= \sum_{k \ne j} 
\E[|\barz_k \barw_j- \mu_{kj}|^2] + \sum_{k \ne j} 
\E[(\barz_k \barw_j- \mu_{kj})(\barz_j \barw_k- \mu_{jk})]\\
&\quad+\sum_{k} \sum_{\substack{j_1 \ne j_2 \\ j_1,j_2 \ne k}}
\E\!\left[
(\barz_k \barw_{j_1}-\mu_{kj_1})
(\barz_k \barw_{j_2}-\mu_{kj_2})
\right] \\
&\quad +
\sum_{j} \sum_{\substack{k_1 \ne k_2 \\ k_1,k_2 \ne j}}
\E\!\left[
(\barz_{k_1}\barw_j-\mu_{k_1 j})
(\barz_{k_2}\barw_j-\mu_{k_2 j})
\right] \\
&\quad +
\sum_{\substack{k_1, k_2,j_1, j_2\\ \text{all different}}}
\E\!\left[
(\barz_{k_1}\barw_{j_1}-\mu_{k_1 j_1})
(\barz_{k_2}\barw_{j_2}-\mu_{k_2 j_2})
\right]\\
&=
\sum_{k \ne j}
\Big(
\E(\barz_k^2)\E(|\barw_j|^2)
- |\mu_{kj}|^2
\Big) + \sum_{k \ne j}
\E(|\barz_k \barw_{k}|)
\E(|\barz_j \barw_{j}|)
- \mu_{kj}\mu_{jk} \\
&\quad+ \sum_{k}
\sum_{\substack{j_1 \ne j_2 \\ j_1,j_2 \ne k}}
\Big(
\E(\barz_k^2)
\E(\barw_{j_1})
\E(\barw_{j_2})
- \mu_{kj_1}\mu_{kj_2}
\Big) \\
&\quad+ \sum_{j}
\sum_{\substack{k_1 \ne k_2 \\ k_1,k_2 \ne j}}
\Big(
\E(\barz_{k_1})
\E(\barz_{k_2})
\E(\barw_j^2)
- \mu_{k_1 j}\mu_{k_2 j}
\Big)\\
&\le
\sum_{k \ne j} \sigma^4
+ \sum_{k \ne j} 2\sigma^4
+ \sum_{k} \sum_{\substack{j_1 \ne j_2 \\ j_1,j_2 \ne k}}
\sigma^4
+ \sum_{j} \sum_{\substack{k_1 \ne k_2 \\ k_1,k_2 \ne j}}
\sigma^4\\
&=(3n(n-1)+2n(n-1)(n-2))\sigma^4\;.
\end{aligned}
\end{equation}
Where in the second equality we use that $\barz_k$ and $\barw_j$ are independent for $k\neq j$.
In the last inequality in (\ref{eq: first term rhs}) we bounded the terms of the form $\E(|\barz_k\barw_k|)$ by $\sigma^2$. A slightly tighter bound is given by $\min\{u,\sigma^2\}$, but we omit it here to avoid cluttering. 
For the second term on the right-hand side of (\ref{eq: basic ineq}), we have that
\begin{equation}\label{eq: second term rhs}
\begin{aligned}
   \left|\sum_{k\neq j}\E(E_kW_j)-\mu_{k,j}\right|
    &\leq\sum_{k\neq j} \E\left(|E_kW_j|\1_{\{|E_k|>\beta\sqrt{k}\;,\;|W_j|\leq\beta\sqrt{j}\}}\right)\\
    &\quad +\sum_{k\neq j} \E\left(|E_kW_j|\1_{\{|E_k|\leq\beta\sqrt{k},\;|W_j|>\beta\sqrt{j}\}}\right)\\
    &\quad +\sum_{k\neq j} \E\left(|E_kW_j|\1_{\{|E_k|>\beta\sqrt{k},\;|W_j|>\beta\sqrt{j}\}}\right)\\
    &\leq\sum_{k\neq j} \sigma\, \E\left(|E_k|\1_{\{|E_k|>\beta\sqrt{k}\}}\right) 
    +\sigma\,\E\left(|W_j|_{|W_j|>\beta\sqrt{j}\}}\right)\\
    & \quad+\sum_{k\neq j}\E\left(|E_kW_j|_{\{|E_k|>\beta\sqrt{k},\;|W_j|>\beta\sqrt{j}\}}\right)\\
    &\leq \sum_{k\neq j} \frac{\sigma\,\E\left(|E_k|^2\right)}{\beta\sqrt{k}} 
    +\frac{\sigma\,\E\left(|W_j|^2\right)}{\beta\sqrt{j}}
    +\frac{\E\left(|E_k|^2\right)\E\left(|W_j|^2\right)}{\beta^2\sqrt{kj}}\\
    &\leq \frac{4\sigma^3\sqrt{n}(n-1)}{{\beta}}
    +\frac{4\sigma^4\sqrt{n(n-1)}}{\beta^2} \leq \left(1+\frac{2\sigma}{\beta\sqrt{n}}\right)\frac{4\sigma^3\sqrt{n}(n-1)}{{\beta}}\;,
\end{aligned}
\end{equation}

where again we use that $E_k$ and $W_j$ are independent for $k\neq j$ and $\E(|E_k|), \E(|W_j|)<\sigma$.
Combining (\ref{eq: basic ineq}), (\ref{eq: first term rhs}), and (\ref{eq: second term rhs}), we obtain that
\begin{align*}
    \E\left|\frac{1}{n(n-1)}\sum_{k\neq j} \barz_k\barw_j -\E(E_1)\E(W_1)\right|^2 &\leq \frac{1}{n^2(n-1)^2}\E\!\left[\left| \sum_{k \ne j} \barz_k \barw_j - \mu_{k,j} \right|^2
\right]\\
&\quad + \frac{1}{n^2(n-1)^2}\left|\sum_{j\neq k}\E(E_jW_k)-\mu_{k,j}\right|^2 \\
&\leq \frac{(3+2(n-2))n(n-1)\sigma^4}{n^2(n-1)^2}+ \left(1+\frac{2\sigma}{{\beta}\sqrt{n}}\right)^2\frac{16\sigma^6n(n-1)^2}{\beta^2 n^2(n-1)^2}\\
&=\left(2+\frac{3}{n-1}+\left(1+\frac{2\sigma}{\beta\sqrt{n}}\right)^2\frac{16\sigma^2}{\beta^2}\right)\frac{\sigma^4}{n}\;.
\end{align*}
\end{proof}
We can now combine Lemma \ref{lemm: generic convergence in espectation} and Lemma \ref{lem: L2 conv of covariance estimator} to obtain an estimate for the expected squared error of the gradient estimator defined in (\ref{eq:clipped_grad_est}).
\begin{corollary}\label{cor: convergence in expectation of estimator}
    Assume Setting \ref{set:optimization} and Assumptions \ref{set: SGD assumptions}(i-iii). Then, for any $\beta>0$
    \begin{align*}
        \E\left(\left|G_n(\theta,u,\beta)-\nabla_\theta L(\theta)\right|^2\right)&\leq 64u^2n^{\frac{2(1-p)}{p}}+\left(2+\frac{3}{n-1}+\left(1+\frac{2\sigma}{\beta\sqrt{n}}\right)^2\frac{16\sigma^2}{\beta^2}\right)\frac{8\sigma^4}{n}\\
    &= O\left(n^{\frac{2(1-p)}{p}}\right)\;.
    \end{align*}
\end{corollary}
\begin{proof}
    Setting $Y_k = E_kW_k$ we have by Lemma \ref{lemm: generic convergence in espectation} and Lemma \ref{lem: L2 conv of covariance estimator} that
    \begin{align*}
        \E\left(\left|G_n(\theta,u,\beta)-\nabla_\theta L(\theta)\right|^2\right) &\leq 2 \E\left|\frac{2}{n}\sum_{k= 1}^n\left(\gamma_{\sigma,p,k}(Y_k)-\E(Y_1)\right)\right|^2\\
        &\quad +2\E\left|\frac{2}{n(n-1)} \sum_{k \neq j} \left(\gamma_{\beta,2,k}(E_k)\gamma_{\beta,2,k}(W_j)- \E(E_1)\E(W_1)\right)\right|^2\\
        &\leq 64u^2n^{\frac{2(1-p)}{p}}+\left(2+\frac{3}{n-1}+\left(1+\frac{2\sigma}{\sqrt{\beta}n}\right)^2\frac{16\sigma^2}{\beta^2}\right)\frac{8\sigma^4}{n} \;.
    \end{align*}
\end{proof}
\begin{remark}
    Note that if no clipping is applied to the second term in (\ref{eq:clipped_grad_est}), i.e. if ``$\beta = \infty$", then the error bound in Corollary \ref{cor: convergence in expectation of estimator} becomes
    \begin{align*}
        \E\left(\left|G_n(\theta,\alpha,\beta)-\nabla_\theta L(\theta)\right|^2\right)&\leq 64u^2n^{-\frac{2(1-p)}{p}}+\left(2+\frac{3}{n-1}\right)\frac{8\sigma^4}{n}\;.
    \end{align*}
    In particular, convergence in expectation can be established, even without clipping the second term in (\ref{eq:clipped_grad_est}). However, this is no longer the case when considering convergence with high-probability.
\end{remark}

We can now prove Theorem \ref{thm: convergence in exp of vartiational monte carlo}

\begin{proof}[Proof of Theorem \ref{thm: convergence in exp of vartiational monte carlo}]
By Proposition \ref{prop: gradient bound for iterates}, we know that the iterates generated by the PS-Clip-VMC satisfy
    \begin{align}\label{eq: telescoped sum}
            \sum_{m=1}^M\eta_m|\nabla_\theta L(\theta_m)|^2 &\leq 2\Delta_1+\sum_{m=1}^M\eta_m|\nabla_\theta L(\theta_m)-G_n(\theta_m,\alpha,\beta)|^2.
        \end{align}
        Further, by Corollary \ref{cor: convergence in expectation of estimator} there is a constant $K_{u,\sigma,\beta}>0$, which only depends on $u$, $\sigma$ and $\beta$, such that $\E(|\nabla_\theta L(\theta)-G_n(\theta)|^2\leq K_{u,\sigma,\beta} n^{\frac{2(1-p)}{p}})$, for any $\theta \in \R^d$. Hence, for any $m=1,\ldots,M$ we have that
        \begin{equation}\label{eq: total expectation upper bound}
            \E\left(|\nabla_\theta L(\theta_m)-G_n(\theta_m)|^2\right) = \E\left(\E_{\theta_m}\left(|\nabla_\theta L(\theta_m)-G_n(\theta_m)|^2\right)\right) \leq K_{u,\sigma,\beta}\, n^{\frac{2(1-p)}{p}} \;.
        \end{equation}
         Plugging (\ref{eq: total expectation upper bound}) into (\ref{eq: telescoped sum}), we obtain that
    \begin{align*}
        \sum_{m=1}^M\eta_m\E(|\nabla_\theta L(\theta_m)|^2) \leq 2\Delta_1 + K_{u,\sigma,\beta}\,{n^{\frac{2(1-p)}{p}}} \sum_{m=1}^M\eta_m\;,
    \end{align*}
    and dividing both sides by $\sum_{m=1}^M\eta_m$ concludes the proof.

\end{proof}
\subsection{High Probability Convergence}
Due to the high computational cost of each training run, it is generally desirable to establish \textit{high-probability} convergence guarantees for an optimization algorithm in the context of VMC. That is, one seeks to prove that, with probability at least $1-\delta$, the considered convergence measure depends at most polylogarithmically on $\log(1/\delta)$.

The following theorem provides a high-probability convergence guarantee for PS-Clip-VMC under Assumptions \ref{set: SGD assumptions}.

\begin{theorem}\label{thm: high prob conv of variational monte carlo}
    Assume Settings \ref{set:optimization} and Assumptions \ref{set: SGD assumptions}. Let $\delta \in (0,1)$ and set $\alpha = u/(\log(4/\delta)+1/4)^{1/p}$ and $\beta = \sigma/(\log(4/\delta)+1/4)^{1/2}$. Then, for any $\eta_m <1/C$ and $n \geq \log((4M)/\delta)+1/4$, we have that with probability at least $1-\delta$ the iterates generated by PS-Clip-VMC satisfy
\begin{align*}
        \frac{1}{\sum_{m=1}^M\eta_m}\sum_{m=1}^M\eta_m|\nabla_\theta L(\theta_m)|^2 \leq \frac{2\Delta_1}{\sum_{m=1}^M\eta_m} + 4(7u+94\sigma^2)^2\left(\frac{\log(4M/\delta)+1/4}{n}\right)^{\frac{2(p-1)}{p}}\;.
    \end{align*}
     In particular, if $\eta_m = 1/(2C)$ is constant and $n \geq M^{\frac{p}{2(p-1)}}$ we obtain that 
    \begin{align*}
        \frac{1}{M}\sum_{m=1}^M|\nabla_\theta L(\theta_m)|^2 \leq \frac{\Delta_1 C + 4(7u+94\sigma^2)^2\left({\log(4M/\delta)+1/4}\right)^{\frac{2(p-1)}{p}}}{M}\;.
    \end{align*}
\end{theorem}
\begin{remark}
        Note that, in general, $n \gg \log((4M)/\delta)$, so that the assumption $n \geq \log((4M)/\delta)+1/4$ in Theorem \ref{thm: high prob conv of variational monte carlo} is not a significant restriction.
    \end{remark}
Similarly as in the previous section, we begin by providing a generic high probability upper bound for random variables with a finite $p$-th moment.
\begin{lemma}\label{lemm: general high prob bound}
    Let $Y_1, \ldots, Y_n$ be i.i.d. random variables such that $\E\left|Y_1\right|^{p}\leq u^p<\infty$ for some $p\in (1,2]$, and for any $\delta\in(0,1)$ and $\tilde{\alpha}\geq \sigma$ set $\alpha = \frac{\tilde{\alpha}}{\log(1/\delta)^{1/p}}$. Then, with probability at least $1-\delta e^{1/4}$ we have that
    \begin{equation*}
        \left|\frac{1}{n}\sum_{k=1}^n \gamma_{{\alpha},p,k}(Y_k) - \E(Y_1)\right|\leq 7\tilde{\alpha}\left(\frac{\log(1/\delta)}{n}\right)^{\frac{p-1}{p}}\;.
    \end{equation*}
\end{lemma}
\begin{proof}
For simplicity we set $\bary_k = \gamma_{\tilde{\alpha},p,k}(Y_k)$. The proof of this result follows the same idea as the one of \cite[Lemma~1]{bandits}.
First, note that for every $k=1, \ldots,n$ we have that 
\begin{align*}
    &\mathbb{E}\left(\bary_k -\E(\bary_k )\right) = 0 \;,\\
&|\bary_k -\E(\bary_k )| \leq 2 \tilde{\alpha}\left(\frac{n}{\log(1/\delta)} \right)^{\frac{1}{p}}
\quad \text{and} \\ 
&\mathbb{E}\big[|\bary_k -\E(\bary_k )|^2\big] \leq u^{p}\tilde{\alpha}^{2-p}\left(\frac{n}{\log(1/\delta)}\right)^\frac{2-p}{p}\;.
\end{align*}
Hence, using the second inequality in Lemma~\ref{lem: vector berstein inequality} we obtain that with probability at least $1-\delta e^{1/4}$
\begin{equation}\label{eq: generic high prob bound in proof}
    \begin{aligned}
\left|\frac{1}{n}\sum_{k=1}^n \bary_k - \E(\bary_k)\right|
&\leq
2 \tilde{\alpha}\left(\frac{n}{\log(1/\delta)} \right)^{\frac{1}{p}}\frac{ \log(1/\delta)}{n}+\sqrt{8u^{p}\tilde{\alpha}^{2-p}\left(\frac{n}{\log(1/\delta)}\right)^\frac{2-p}{p}\frac{\log(1/\delta)}{n}}\\
&=\tilde{\alpha}\left(2+\sqrt{8}(u\tilde{\alpha}^{-1})^{p/2}\right)\left(\frac{\log(1/\delta)}{n}\right)^{\frac{p-1}{p}}
<5\tilde{\alpha}\left(\frac{\log(1/\delta)}{n}\right)^{\frac{p-1}{p}}\;.
\end{aligned}
\end{equation}
Combining this, with (\ref{eq: expectation bias}) we obtain that with probability at least $1-\delta e^{1/4}$
\begin{align*}
\left|\E(Y_1) 
- \frac{1}{n}\sum_{k=1}^n \bary_k\right|
&\leq
\left|\frac{1}{n}\sum_{k=1}^n 
\Big(\E(Y_k) - \E(\bary_k)\Big)\right|
+\left|
\frac{1}{n}\sum_{k=1}^n 
\Big(\E(\bary_k)
- Y_k \gamma_k \Big)\right|
\\
&\leq
\frac{1}{n}\sum_{k=1}^n 
u^p\tilde{\alpha}^{1-p}
\log(1/\delta)^{\frac{p-1}{p}} k^{\frac{1-p}{p}}
+5\tilde{\alpha}\left(\frac{\log(1/\delta)}{n}\right)^{\frac{p-1}{p}}
\\
&\leq
\frac{2}{n} 
u^p\tilde{\alpha}^{1-p}\log(1/\delta)^{\frac{p-1}{p}} n^{\frac{1}{p}}
+5\tilde{\alpha}\left(\frac{\log(1/\delta)}{n}\right)^{\frac{p-1}{p}}
\\
&\leq
2u^p\tilde{\alpha}^{1-p}\log(1/\delta)^{\frac{p-1}{p}} n^{\frac{1-p}{p}}
+5\tilde{\alpha}\left(\frac{\log(1/\delta)}{n}\right)^{\frac{p-1}{p}}
\\
&\leq 7\tilde{\alpha}\left(\frac{\log(1/\delta)}{n}\right)^{\frac{p-1}{p}} \;,
\end{align*}
where in the second inequality we used (\ref{eq: expectation bias}) to estimate the first term and (\ref{eq: generic high prob bound in proof}) for
the second one.
\end{proof}

We can now provide a high probability error bound for the gradient estimator defined in Setting \ref{set:optimization}.

\begin{lemma}\label{lemm: high prob bound for estimator}
    Assume Settings \ref{set:optimization} and Assumptions \ref{set: SGD assumptions}(i-iii). For any $\delta \in (0,1)$, set $\alpha = u/\log(4/\delta)^{1/p}$ and $\beta = \sigma/\log(4/\delta)^{1/2}$. Then, with probability at least $1-e^{1/4}\delta$
    \begin{align*}
    \frac{1}{2}\left|G_n(\theta,\alpha,\beta)-\nabla_\theta L(\theta)\right| &\leq 7u\left(\frac{\log(4/\delta)}{n}\right)^{\frac{p-1}{p}} +24\sigma^2\sqrt{\frac{\log(4/\delta)}{n}} + 58\sigma^2\left(\frac{\log(4/\delta)}{n}\right)+\frac{12\sigma^2}{n}\\
    &=O\left(\left(\frac{\log(4/\delta)}{n}\right)^{\frac{p-1}{p}}\right) \;.
\end{align*}
\end{lemma}
\begin{proof}
Let $Y_k = E_kW_k$ and define the clipped random variables $\barz_k =\gamma_{\beta,2,k}(E_k)$, $\barw_k = \gamma_{\beta,2,k}(W_k)$ and $\bary_k = \gamma_{\alpha,p,k}(Y_k)$. Then, with this notation we have that
    \begin{equation}\label{eq: rewrite of estimator with rv}
    \frac{1}{2}\left(G_n(\theta,\alpha,\beta)-\nabla_\theta L(\theta)\right)=\left(\frac{1}{n} \sum_{k=1}^n \bary_k-\E(Y_1)\right) +\left(\E(E_1)\E(W_1)-\frac{1}{n^2-n} \sum_{k \neq j} \barz_k\barw_j\right)\;.
\end{equation}

    Let $\rho = \delta/4$, so that $\alpha = u/\log(1/\rho)^{1/p}$ and $\beta = \sigma/\log(1/\rho)^{1/2}$. Then, for the first term on the right-hand side of (\ref{eq: rewrite of estimator with rv}), we know by Lemma \ref{lemm: general high prob bound} that with probability at least $1-e^{1/4}\rho$
    \begin{equation}\label{eq: concentration ineq. for p-rv}
        \left|\frac{1}{n}\sum_{k=1}^n Y_k \gamma_{{\alpha},p,k}(Y_k) - \E(Y)\right|\leq 7u\left(\frac{\log(1/\rho)}{n}\right)^{\frac{p-1}{p}}\;.
    \end{equation}

To estimate the second term in (\ref{eq: rewrite of estimator with rv}) we rewrite it as follows:
\begin{equation}\label{eq:decomposition of total error}
    \begin{aligned}
        \E[E_1]\E[W_1]-\frac{1}{n(n-1)} \sum_{j\neq k}\Bar{E}_j\Bar{W}_k
        &= \underbrace{\frac{1}{n(n-1)} \sum_{j\neq k}\E[E_1]\E[W_1]-\E[\Bar{E}_j]\E[\Bar{W}_k]}_{\mathrm{I}}\\
        &\quad-\underbrace{\frac{1}{n(n-1)} \sum_{j, k} \barz_j\barw_k-\E[\Bar{E}_j]\E[\Bar{W}_k]}_{\mathrm{II}}\\
        &\quad+\underbrace{\frac{1}{n(n-1)} \sum_{j}\barz_j\barw_j-\E[\Bar{Z}_j]\E[\Bar{W}_j]}_{\mathrm{III}}\;.
    \end{aligned}
\end{equation}
Note that by independence, we have that $\E[E_1]\E[W_1]  = \E(E_jW_k)$ for any $j\neq k$. Hence, using (\ref{eq: second term rhs}) we obtain the following bound for the term $\mathrm{I}$ on the right-hand side of (\ref{eq:decomposition of total error}):

    \begin{equation}\label{eq: bound for I}
    (\mathrm{I})= \frac{1}{n(n-1)}\sum_{j\neq k}\E[E_jW_k]-E[\Bar{E}_j\Bar{W}_k]\leq\left(1+\frac{2\sigma}{\beta\sqrt{n}}\right)\frac{4\sigma^3}{{\beta\sqrt{n}}} = 
    \left(1+2\sqrt{\frac{\log(1/\rho)}{n}}\right)\sqrt{\frac{\log(1/\rho)}{n}}4\sigma^2\;.
\end{equation}
    To estimate (II) we rewrite the sum as
    \begin{align*}
        \sum_{j, k =1}^n \barz_j\barw_k-\E[\Bar{E}_j\Bar{W}_k] &=
\sum_{k=1}^n \E[\Bar{W}_k]\sum_{j=1}^n (\Bar{E}_j - \E[\Bar{E}_j])
+ \sum_{j=1}^n \E[\Bar{E}_j]\sum_{k=1}^n (\Bar{W}_k - \E[\Bar{W}_k])\\
&\quad+\sum_{j=1}^n (\Bar{E}_j - \E[\Bar{E}_j]) 
  \sum_{k=1}^n (\Bar{W}_k - \E[\Bar{W}_k])\;.
    \end{align*}
Note that the random variables $\barz_j -\E[\barz_j]$ and $\barw_j -\E[\barw_j]$ satisfy
\begin{align*}
    &\mathbb{E}\left(\barz_j-\E(\barz_j)\right) = 0 \;,\quad 
    |\barz_j-\E(\barz_j)| \leq 2 \sigma\left(\frac{n}{\log(1/\rho)} \right)^{\frac{1}{2}} \;,
\quad \mathbb{E}\big[|\barz_j-\E(\barz_j)|^2\big] \leq \sigma^{2}\left(\frac{n}{\log(1/\rho)}\right)^\frac{1}{2}
\end{align*}
\begin{align*}
    &\mathbb{E}\left(\barw_j-\E(\barw_j)\right) = 0 \;,\quad 
    |\barw_j-\E(\barw_j)| \leq 2 \sigma\left(\frac{n}{\log(1/\rho)} \right)^{\frac{1}{2}} \;,
\quad \mathbb{E}\big[|\barw_j-\E(\barw_j)|^2\big] \leq \sigma^{2}\left(\frac{n}{\log(1/\rho)}\right)^\frac{1}{2}.
\end{align*}
Hence, using the same steps as in (\ref{eq: generic high prob bound in proof}), we obtain that with probability $1-2e^{1/4}\rho$
\begin{equation}\label{eq: bound for II}
\begin{aligned}
    (\text{II}) &=\frac{1}{n(n-1)}\Bigg(\sum_{k=1}^n \E(\Bar{W}_k)\sum_{j=1}^n (\Bar{E}_j - \E(\Bar{E}_j))
+ \sum_{j=1}^n \E(\Bar{E}_j)\sum_{k=1}^n (\Bar{W}_k - \E(\Bar{W}_k))\\
&\quad+\sum_{j=1}^n (\Bar{E}_j - \E(\Bar{E}_j)) 
  \sum_{k=1}^n (\Bar{W}_k - \E(\Bar{W}_k))\Bigg)\\
&\leq\frac{1}{(n-1)}\left(\sum_{k=1}^n E(|\Bar{W}_k|) 5\sigma\left(\frac{\log(1/\rho)}{n}\right)^{\frac{1}{2}}
+ \sum_{j=1}^n\E(|\Bar{E}_j|) 5\sigma\left(\frac{\log(1/\rho)}{n}\right)^{\frac{1}{2}}\right)\\
&\quad+\frac{n^2}{n(n-1)}25\sigma^2\left(\frac{\log(1/\rho)}{n}\right)\\
&\leq 20\sigma^2\left(\frac{\log(1/\rho)}{n}\right)^{1/2}+50\sigma^2\left(\frac{\log(1/\rho)}{n}\right).
\end{aligned}
\end{equation}
To estimate (III), we note that the term can be rewritten as follows
\begin{equation}
\begin{aligned}
    \mathrm{(III)} &= \frac{1}{n(n-1)} \sum_{j=1}^n 
    \barz_j\barw_j-\E(\Bar{E}_j\Bar{W}_j) + \frac{1}{n(n-1)} \sum_{j=1}^n \E(\Bar{E}_j\Bar{W}_j) -\E(\Bar{E}_j)\E(\Bar{W}_j).
\end{aligned}
\end{equation}
The second term in (\ref{eq: bound for III}) is deterministic and easily seen to be bounded by $\sigma^2/(n-1) \leq 2\sigma^2/n$.
For the first term, we note that it satisfies
\begin{align*}
    &\mathbb{E}\left(\barz_j\barw_j-\E(\barz_j\barw_j)\right) = 0 \\
    &|\barz_j\barw_j-\E(\barz_j\barw_j)| \leq 2 \sigma^2\left(\frac{n}{\log(1/\rho)} \right)\\
& \mathbb{E}\big[|\barz_j\barw_j-\E(\barz_j\barw_j)|^2\big] \leq \sigma^4\left(\frac{n}{\log(1/\rho)}\right)
\end{align*}
Hence, similarly as before, we obtain that with probability at least $1-\rho e^{1/4}$
\begin{equation}\label{eq: bound for III}
    \frac{1}{n(n-1)} \left|\sum_{j=1}^n 
    \barz_j\barw_j-\E(\Bar{E}_j\Bar{W}_j)\right| \leq \frac{1}{n-1}\left(2\sigma^2+\sqrt{8}\sigma^2\right)\leq \frac{10\sigma^2}{n}
\end{equation}
Hence, plugging (\ref{eq: concentration ineq. for p-rv}), (\ref{eq: bound for I}), (\ref{eq: bound for II}), and (\ref{eq: bound for III})  into (\ref{eq: rewrite of estimator with rv}) we obtain that with probability $1-4e^{1/4}\rho$
\begin{align*}
    \frac{1}{2}\left(G_n(\theta,\alpha,\beta)-\nabla_\theta L(\theta)\right) &\leq 7u\left(\frac{\log(1/\rho)}{n}\right)^{\frac{p-1}{p}} + \left(1+2\sqrt{\frac{\log(1/\rho)}{n}}\right)\sqrt{\frac{\log(1/\rho)}{n}}4\sigma^2\\
    &\quad+ 20\sigma^2\sqrt{\frac{\log(1/\rho)}{n}}+50\sigma^2\left(\frac{\log(1/\rho)}{n}\right) +\frac{12\sigma^2}{n}\\
    &= 7u\left(\frac{\log(1/\rho)}{n}\right)^{\frac{p-1}{p}} +24\sigma^2\sqrt{\frac{\log(1/\rho)}{n}} + 58\sigma^2\left(\frac{\log(1/\rho)}{n}\right)+\frac{12\sigma^2}{n}
\end{align*}
The statement in the Lemma follows by recalling that $\rho = \delta/4$.
\end{proof}

\begin{corollary}\label{cor: high prob bound of maximum}
    Let $\theta_1, \ldots, \theta_M$ be arbitrary weights,  $\delta \in (0,1)$ and $n \geq \log((4M)/\delta)+1/4$. Set $\alpha = u/(\log(4/\delta)+1/4)^{1/p}$ and $\beta = \sigma/(\log(4/\delta)+1/4)^{1/2}$. Then, under the same conditions of Lemma \ref{lemm: high prob bound for estimator}, we have that
    \begin{equation*}
        \max_{m=1,\ldots,M}\left|G_n(\theta_m,\alpha,\beta)-\nabla_\theta L(\theta)\right| \leq 2(7u+94\sigma^2)\left(\frac{\log(4M/\delta)+1/4}{n}\right)^{\frac{p-1}{p}}\;.
    \end{equation*}
    \begin{proof}
        Follows immediately by applying union bound to Lemma \ref{lemm: high prob bound for estimator}.
    \end{proof}
\end{corollary}

The proof of Theorem \ref{thm: high prob conv of variational monte carlo} now follows using similar steps to those of its in Expectation counterpart.
\begin{proof}[Proof of Theorem \ref{thm: high prob conv of variational monte carlo}]
    By Proposition \ref{prop: gradient bound for iterates} we know that the iterates of the algorithm satisfy
    \begin{align*}
            \sum_{m=1}^M\eta_m|\nabla_\theta L(\theta_m)|^2 \leq 2\Delta_1+\sum_{m=1}^M\eta_m|\nabla_\theta L(\theta_m)-G_n(\theta_m,\alpha,\beta)|^2 \;.
    \end{align*}
    Hence, applying Corollary \ref{cor: high prob bound of maximum} we obtain that with probability at least $1-\delta$
    \begin{equation*}
        \sum_{m=1}^M\eta_m|\nabla_\theta L(\theta_m)|^2 \leq 2\Delta_1+4(7u+94\sigma^2)^2\left(\frac{\log(4M/\delta)+1/4}{n}\right)^{\frac{2(p-1)}{p}}\sum_{m=1}^M\eta_m \;,
    \end{equation*}
    and the result follows by dividing both sides by $\sum_{m=1}^M\eta_m$.
\end{proof}
\section{Experiments}\label{sec: Experiments}

We empirically tested the robustness of PS-Clip-VMC and compared its performance to the standard practice of only clipping the local energy. To this end, we use the FermiNet implementation from \cite{googledeepmind_ferminet} with 256 hidden dimensions and 16 Slater determinants.

Although our theoretical results are based on a clipping threshold which increases with each sample, we found that a constant clipping threshold, analogue to the one commonly used for local energy clipping was more effective. Specifically, let $n \in \mathbb{N}$ be the batch size (in our case 2048), $E_1^\theta, \ldots, E_n^\theta$ be samples of the local energy and $W_1^\theta, \ldots, W_n^\theta$ corresponding samples of $\nabla_\theta\log|\psi_\theta|$. Further, let $\mu = \mu(E_1^\theta, \ldots, E_n^\theta) = \frac{1}{n}\sum_{i=1}^n E_i^\theta$, $\sigma = \sigma(E_1^\theta, \ldots, E_n^\theta) = \frac{1}{n}\sum_{i=1}^n |E_i^\theta-\mu|$ and
$\mu' = \mu'(W_1^\theta, \ldots, W_n^\theta) = \frac{1}{n}\sum_{i=1}^n E_i^\theta$, $\sigma' = \sigma'(W_1^\theta, \ldots, W_n^\theta) = \frac{1}{n}\sum_{i=1}^n |W_i^\theta-\mu'|$.
Then, the energy gradient estimator for PS-Clip-VMC used in our experiments is given by 
\begin{equation}\label{eq: empirical estimator}
    G(\theta,\alpha,\beta) = \frac{1}{n}\sum_{i=1}^n \left(\clip_{[\mu-\alpha\sigma,\mu+\alpha\sigma]}\left(E_i^\theta\right)-\frac{1}{n}\sum_{i=1}^n \clip_{[\mu-\alpha\sigma,\mu+\alpha\sigma]}\left(E_i^\theta\right)\right)\min \left\{1, \frac{\mu'+\beta\sigma'}{\left|W_i^\theta\right|} \right\}W_i^\theta \;,
\end{equation}
where 
\begin{equation*}
    \clip_{[a,b]}(x) =
\begin{cases}
a & \text{if } x < a \\
x & \text{if } a \le x \le b \\
b & \text{if } x > b
\end{cases} \;.
\end{equation*}
For both methods, we use the default hyperparameters from the FermiNet GitHub repository, except for the number of decorrelation MCMC steps, which we increase to 30. The model was trained for $2\cdot 10^5$ iterations with a batch size of 2048; note that this corresponds to half of the batch size normally used to train FermiNet. We refer to Table~\ref{tab:hyperparams} for a detailed list of the hyperparameters used. To efficiently compute per-sample gradient norms, we use Google's \textit{JAX Privacy} API (\url{https://github.com/google-deepmind/jax_privacy}). To facilitate reproducibility, we created a fork of Google's FermiNet repository (\url{https://github.com/mdsunivie/ps_clip_ferminet}), available at \url{https://github.com/nobiledavide/PS_Clip_ferminet}, in which we implemented the per-sample gradient clipping method used in our experiments.

Table~\ref{tab: final energies} compares the final energies obtained using only standard (centered) local energy clipping with those obtained using the additional PS-gradient clipping. We observe that PS-Clip-VMC outperforms the standard method in both instances, with the performance difference for Argon being particularly significant. The poor performance of the standard method on Argon can be explained by examining Figure~\ref{fig: Ferminet training trajectories}. While PS-Clip-VMC converges steadily, the training trajectory of the standard method is more unstable. In particular, it exhibits a cusp around step 55,000, after which it is unable to recover a lower energy.

\begin{table}[h]
\centering
\begin{tabular}{lll}
\toprule
\textbf{} & \textbf{Parameter} & \textbf{Value} \\
\midrule
\multirow{5}{*}{Training}
& Learning rate at time $t$ & $l_{r0}(1 + t/t_0)^{-1}$ \\
& Initial learning rate & $l_{r0} = 0.05$ \\
& Learning rate decay & $t_0 = 10^{5}$ \\
& Local energy clipping & $\alpha = 5.0$ \\
& PS-Gradient clipping & $\beta = 5.0$ or ``$\beta = \infty$"\\
\midrule
\multirow{3}{*}{Pretraining}
& Pretraining optimizer & Adam \\
& Pretraining iterations & $10^4$ \\
& Pretraining basis set & ccpvdz \\
\midrule
\multirow{2}{*}{MCMC}
& Batch size & 2048 \\
& Decorrelation steps & 30 \\
\midrule

\multirow{2}{*}{KFAC}
& Norm constraint & $10^{-3}$ \\
& Damping & $10^{-3}$ \\
\bottomrule
\end{tabular}
\caption{Default hyperparameters used to train the network. $\beta = \infty$ corresponds to the standard practice of only clipping the local energy. See also \eqref{eq: empirical estimator} for the definition of the parameters $\alpha$ and $\beta$.}
\label{tab:hyperparams}
\end{table}
\section{Conclusion}
In this work, we provided a rigorous analysis of the moments properties of the stochastic optimization problem arising in variational Monte Carlo for the electronic Schr\"odinger equation. We demonstrated that for commonly used ansatz classes, such as Slater--Jastrow wave functions with variable-exponent Slater-type orbitals, the random variable used to estimate the stochastic gradient generally does not possess moments of order greater than 3/2. On the other hand, we showed that for any analytic, compactly supported ansatz, this random variable does possesses moments of order at least 5/4.

Motivated by these observations, we proposed a robust optimization algorithm that guarantees convergence of VMC both in expectation and with high probability in the weak-moment regime. We validated our theoretical results through numerical experiments using FermiNet and demonstrated that our method can be implemented efficiently using a differential privacy API, allowing it to be incorporated into existing VMC implementations with little additional overhead.

\newpage
\printbibliography

@article{abrahamsen2024convergence,
  author       = {Abrahamsen, Nilin and Ding, Zhiyan and Goldshlager, Gil and Lin, Lin},
  title        = {Convergence of variational Monte Carlo simulation and scale‑invariant pre‑training},
  journal      = {Journal of Computational Physics},
  year         = {2024},
  volume       = {513},
  pages        = {113140},
  doi          = {10.1016/j.jcp.2024.113140},
}

@book{teschl2014mathematical,
  title={Mathematical methods in quantum mechanics},
  author={Teschl, Gerald},
  volume={157},
  year={2014},
  publisher={American Mathematical Soc.}
}

@book{wendland2004scattered,
  title={Scattered data approximation},
  author={Wendland, Holger},
  volume={17},
  year={2004},
  publisher={Cambridge university press}
}

@article{lin2023explicitly,
  title={Explicitly antisymmetrized neural network layers for variational Monte Carlo simulation},
  author={Lin, Jeffmin and Goldshlager, Gil and Lin, Lin},
  journal={Journal of Computational Physics},
  volume={474},
  pages={111765},
  year={2023},
  publisher={Elsevier}
}

@incollection{toulouse2016introduction,
  title={Introduction to the variational and diffusion Monte Carlo methods},
  author={Toulouse, Julien and Assaraf, Roland and Umrigar, Cyrus J},
  booktitle={Advances in quantum chemistry},
  volume={73},
  pages={285--314},
  year={2016},
  publisher={Elsevier}
}

@book{szabo2012modern,
  title={Modern quantum chemistry: introduction to advanced electronic structure theory},
  author={Szabo, Attila and Ostlund, Neil S},
  year={2012},
  publisher={Courier Corporation}
}

@article{trail2008heavy,
  title={Heavy-tailed random error in quantum Monte Carlo},
  author={Trail, JR},
  journal={Physical Review E—Statistical, Nonlinear, and Soft Matter Physics},
  volume={77},
  number={1},
  pages={016703},
  year={2008},
  publisher={APS},
  doi = {10.1103/PhysRevE.77.016703}
}

@book{becca2017quantum,
  title={Quantum Monte Carlo approaches for correlated systems},
  author={Becca, Federico and Sorella, Sandro},
  year={2017},
  publisher={Cambridge University Press}
}

@article{nobile2026robust,
  title={Robust and Fast Training via Per-Sample Clipping},
  author={Nobile, Davide and Grohs, Philipp},
  journal={arXiv preprint arXiv:2605.02701},
  year={2026}
}

@book{range1998holomorphic,
  title={Holomorphic functions and integral representations in several complex variables},
  author={Range, R Michael},
  volume={108},
  year={1998},
  publisher={Springer Science \& Business Media}
}

@article{pfau2020ab,
  title={Ab initio solution of the many-electron Schr{\"o}dinger equation with deep neural networks},
  author={Pfau, David and Spencer, James S and Matthews, Alexander GDG and Foulkes, W Matthew C},
  journal={Physical review research},
  volume={2},
  number={3},
  pages={033429},
  year={2020},
  publisher={APS},
  doi = {10.1103/PhysRevResearch.2.033429}
}

@article{aldaz2009bernstein,
  title={Bernstein operators for exponential polynomials},
  author={Aldaz, JM and Kounchev, Ognyan and Render, Hermann},
  journal={Constructive Approximation},
  volume={29},
  number={3},
  pages={345--367},
  year={2009},
  publisher={Springer},
  doi = {10.1007/s00365-008-9010-6}
}

@book{pinkus2012n,
  title={N-widths in Approximation Theory},
  author={Pinkus, Allan},
  year={2012},
  publisher={Springer Science \& Business Media}
}

@book{helgaker2013molecular,
  title={Molecular electronic-structure theory},
  author={Helgaker, Trygve and Jorgensen, Poul and Olsen, Jeppe},
  year={2013},
  publisher={John Wiley \& Sons}
}

@book{kato:76:perturbation,
  author    = {Kato, Tosio},
  title     = {Perturbation theory for linear operators},
  publisher = {Springer-Verlag, Berlin-New York},
  series    = {Grundlehren der Mathematischen Wissenschaften, Band 132},
  edition   = {Second},
  year      = {1976},
  pages     = {xxi+619},
  mrclass   = {47-XX},
  mrnumber  = {0407617}
}

@book{korevaar2017several,
  title={Several complex variables},
  author={Korevaar, Jacob and Wiegerinck, Jan},
  year={2017},
  publisher={Korteweg-de Vries Institute for Mathematics Amsterdam}
}

@book{gustafson2003mathematical,
  title={Mathematical concepts of quantum mechanics},
  author={Gustafson, Stephen J and Sigal, Israel Michael and Sigal, Israel Michael and Physicien, Isra{\"e}l and Sigal, Israel Michael and Physicist, Israel},
  volume={33},
  year={2003},
  publisher={Springer}
}

@incollection{agmon2006bounds,
  title={Bounds on exponential decay of eigenfunctions of Schr{\"o}dinger operators},
  author={Agmon, Shmuel},
  booktitle={Schr{\"o}dinger Operators: Lectures given at the 2nd 1984 Session of the Centro Internationale Matematico Estivo (CIME) held at Como, Italy, Aug. 26--Sept. 4, 1984},
  pages={1--38},
  year={2006},
  publisher={Springer}
}

@article{grohs2021stable,
  title={Stable Gabor phase retrieval for multivariate functions},
  author={Grohs, Philipp and Rathmair, Martin},
  journal={Journal of the European Mathematical Society},
  volume={24},
  number={5},
  pages={1593--1615},
  year={2021},
  doi = {10.4171/JEMS/1114}
}

@article{kato1957eigenfunctions,
  title={On the eigenfunctions of many-particle systems in quantum mechanics},
  author={Kato, Tosio},
  journal={Communications on Pure and Applied Mathematics},
  volume={10},
  number={2},
  pages={151--177},
  year={1957},
  publisher={Wiley Online Library},
  doi = {10.1002/cpa.3160100201}
}

@article{renaud2025qmctorch,
  title={QMCTorch: Molecular Wave Function with Neural Components for Energy and Force Calculations},
  author={Renaud, Nicolas},
  journal={Methods},
  volume={4},
  pages={4},
  year={2025}
}

@article{hermann2020deep,
  title={Deep-neural-network solution of the electronic Schr{\"o}dinger equation},
  author={Hermann, Jan and Sch{\"a}tzle, Zeno and No{\'e}, Frank},
  journal={Nature Chemistry},
  volume={12},
  number={10},
  pages={891--897},
  year={2020},
  publisher={Nature Publishing Group UK London},
  doi = {10.1038/s41557-020-0544-y}
}

@inproceedings{malgrange1964preparation,
  title={The preparation theorem for differentiable functions},
  author={Malgrange, Bernard},
  booktitle={Differential Analysis, Bombay Colloq},
  pages={203--208},
  year={1964}
}

@article{armegioiu2025functional,
  title={Functional Neural Wavefunction Optimization},
  author={Armegioiu, Victor and Carrasquilla, Juan and Mishra, Siddhartha and M{\"u}ller, Johannes and Nys, Jannes and Zeinhofer, Marius and Zhang, Hang},
  journal={arXiv preprint arXiv:2507.10835},
  year={2025}
}

@article{bandits,
author = {Bubeck, Sebastien and Cesa-Bianchi, Nicolo and Lugosi, Gabor},
title = {Bandits With Heavy Tail},
year = {2013},
issue_date = {November 2013},
volume = {59},
number = {11},
journal = {IEEE Trans. Inf. Theor.},
pages = {7711–7717},
doi = {10.1109/TIT.2013.2277869}
}

@inproceedings{kohler2017subsampled,
author = {Kohler, Jonas Moritz and Lucchi, Aurelien},
title = {Sub-sampled cubic regularization for non-convex optimization},
year = {2017},
booktitle = {Proceedings of the 34th International Conference on Machine Learning - Volume 70},
pages = {1895–1904},
numpages = {10},
series = {ICML'17}
}

@book{hirsch2012differential,
  title={Differential topology},
  author={Hirsch, Morris W},
  year={2012},
  publisher={Springer Science \& Business Media}
}

@article{chakravorty1993ground,
  author  = {Chakravorty, Subhas J. and Gwaltney, Steven R. and Davidson, Ernest R. and Parpia, Farid A. and Fischer, Charlotte Froese},
  title   = {Ground-state correlation energies for atomic ions with 3 to 18 electrons},
  journal = {Physical Review A},
  volume  = {47},
  number  = {5},
  pages   = {3649--3670},
  year    = {1993},
  doi = {10.1103/PhysRevA.47.3649}
}

@misc{spencer2020better,
  title = {Better, Faster Fermionic Neural Networks},
  author = {Spencer, James S. and Pfau, David and Botev, Aleksandar and Foulkes, W. M. C.},
  year = {2020},
  eprint = {2011.07125},
  archivePrefix = {arXiv}
}

@misc{pang2022universalAntisymmetry,
  title = {{$O(N^2)$} Universal Antisymmetry in Fermionic Neural Networks},
  author = {Pang, Tianyu and Yan, Shuicheng and Lin, Min},
  year = {2022},
  eprint = {2205.13205},
  archivePrefix = {arXiv}
}

@inproceedings{vonGlehn2023psiformer,
  title = {A Self-Attention Ansatz for Ab-initio Quantum Chemistry},
  author = {von Glehn, Ingrid and Spencer, James S. and Pfau, David},
  booktitle = {The Eleventh International Conference on Learning Representations},
  year = {2023},
  eprint = {2211.13672},
  archivePrefix = {arXiv}
}

@article{scherbela2022weightSharing,
  title = {Solving the electronic {Schr{\"o}dinger} equation for multiple nuclear geometries with weight-sharing deep neural networks},
  author = {Scherbela, Michael and Reisenhofer, Rafael and Gerard, Leon and Marquetand, Philipp and Grohs, Philipp},
  journal = {Nature Computational Science},
  volume = {2},
  number = {5},
  pages = {331--341},
  year = {2022},
  doi = {10.1038/s43588-022-00228-x},
  eprint = {2105.08351},
  archivePrefix = {arXiv}
}

@misc{scherbela2023foundationPreprint,
  title = {Towards a Foundation Model for Neural Network Wavefunctions},
  author = {Scherbela, Michael and Gerard, Leon and Grohs, Philipp},
  year = {2023},
  eprint = {2303.09949},
  archivePrefix = {arXiv},
  note = {Preprint title/version of the transferable molecular neural-wavefunction work}
}

@article{scherbela2024transferableMolecules,
  title = {Towards a transferable fermionic neural wavefunction for molecules},
  author = {Scherbela, Michael and Gerard, Leon and Grohs, Philipp},
  journal = {Nature Communications},
  volume = {15},
  pages = {120},
  year = {2024},
  doi = {10.1038/s41467-023-44216-9}
}

@inproceedings{scherbela2023budget,
  title = {Variational {Monte Carlo} on a Budget: Fine-tuning pre-trained Neural Wavefunctions},
  author = {Scherbela, Michael and Gerard, Leon and Grohs, Philipp},
  booktitle = {Advances in Neural Information Processing Systems},
  volume = {36},
  year = {2023},
  eprint = {2307.09337},
  archivePrefix = {arXiv},
}

@incollection{gerard2024deepLearningVMCChapter,
  title = {Deep learning variational {Monte Carlo} for solving the electronic {Schr{\"o}dinger} equation},
  author = {Gerard, Leon and Grohs, Philipp and Scherbela, Michael},
  booktitle = {Numerical Analysis Meets Machine Learning},
  series = {Handbook of Numerical Analysis},
  volume = {25},
  pages = {231--292},
  publisher = {North-Holland},
  year = {2024},
  doi = {10.1016/bs.hna.2024.05.010}
}

@article{gerard2025transferableSolids,
  title = {Transferable neural wavefunctions for solids},
  author = {Gerard, Leon and Scherbela, Michael and Sutterud, Halvard and Foulkes, W. M. C. and Grohs, Philipp},
  journal = {Nature Computational Science},
  volume = {5},
  pages = {1147--1157},
  year = {2025},
  doi = {10.1038/s43588-025-00872-z},
  eprint = {2405.07599},
  archivePrefix = {arXiv}
}

@misc{scherbela2025fire,
  title = {Accurate Ab-initio Neural-network Solutions to Large-Scale Electronic Structure Problems},
  author = {Scherbela, Michael and Gao, Nicholas and Grohs, Philipp and G{\"u}nnemann, Stephan},
  year = {2025},
  eprint = {2504.06087},
  archivePrefix = {arXiv},
}

@inproceedings{gao2022pesnet,
  title = {Ab-Initio Potential Energy Surfaces by Pairing {GNNs} with Neural Wave Functions},
  author = {Gao, Nicholas and G{\"u}nnemann, Stephan},
  booktitle = {International Conference on Learning Representations},
  year = {2022},
  eprint = {2110.05064},
  archivePrefix = {arXiv},
}

@inproceedings{gao2023samplingFree,
  title = {Sampling-free Inference for Ab-Initio Potential Energy Surface Networks},
  author = {Gao, Nicholas and G{\"u}nnemann, Stephan},
  booktitle = {The Eleventh International Conference on Learning Representations},
  year = {2023},
  eprint = {2205.14962},
  archivePrefix = {arXiv},
}

@inproceedings{gao2023generalizing,
  title = {Generalizing Neural Wave Functions},
  author = {Gao, Nicholas and G{\"u}nnemann, Stephan},
  booktitle = {Proceedings of the 40th International Conference on Machine Learning},
  series = {Proceedings of Machine Learning Research},
  volume = {202},
  pages = {10708--10726},
  publisher = {PMLR},
  year = {2023},
  eprint = {2302.04168},
  archivePrefix = {arXiv},
}

@article{hermann2023abinitioReview,
  title = {Ab initio quantum chemistry with neural-network wavefunctions},
  author = {Hermann, Jan and Spencer, James S. and Choo, Kenny and Mezzacapo, Antonio and Foulkes, W. M. C. and Pfau, David and Carleo, Giuseppe and No{\'e}, Frank},
  journal = {Nature Reviews Chemistry},
  volume = {7},
  pages = {692--709},
  year = {2023},
  doi = {10.1038/s41570-023-00516-8}
}

@inproceedings{gerard2022goldStandard,
  title = {Gold-standard solutions to the {Schr{\"o}dinger} equation using deep learning: How much physics do we need?},
  author = {Gerard, Leon and Scherbela, Michael and Marquetand, Philipp and Grohs, Philipp},
  booktitle = {Advances in Neural Information Processing Systems},
  volume = {35},
  pages = {10282--10294},
  year = {2022},
  eprint = {2205.09438},
  archivePrefix = {arXiv}
}

@online{googledeepmind_ferminet,
  author       = {{Google DeepMind}},
  title        = {FermiNet},
  year         = {2020},
  url          = {https://github.com/google-deepmind/ferminet},
  note         = {GitHub repository},
}

@article{li2024convergenceanalysisstochasticgradient,
  author  = {Li, Tianyou and Chen, Fan and Chen, Huajie and Wen, Zaiwen},
  title   = {Convergence analysis of stochastic gradient descent with MCMC estimators},
  journal = {Science China Mathematics},
  year    = {2026},
}

@misc{wan2026removingnodalsupportmismatchpathologies,
      title={Removing nodal and support-mismatch pathologies in Variational Monte Carlo via blurred sampling}, 
      author={Zhou-Quan Wan and Roeland Wiersema and Shiwei Zhang},
      year={2026},
      eprint={2603.18148},
      archivePrefix={arXiv},
}
\newpage
\appendix
\section{Additional Technical Results}
\begin{lemma}[Vector Bernstein Inequality, see {\cite[Lemma~18]{kohler2017subsampled}}]\label{lem: vector berstein inequality}
Let $X_1, \ldots, X_n$ be independent vector-valued random variables with common dimension $d$, satisfying
\begin{equation*}
    \mathbb{E}[X_i] = 0 \;,
\quad
|X_i| \le c 
\quad \text{and} \quad 
\mathbb{E}\big[|X_i|^2\big] \le \sigma^2,
\end{equation*}
for all $i= 1,\ldots,n$ and some $c,\sigma>0$. Then, for every $\epsilon>0$ we have that
\begin{equation*}
    \mathbb{P}\left(\left|\frac{1}{n}\sum_{i=1}^nX_i\right| \ge \epsilon\right) 
\leq
\exp\left(\frac{-n\epsilon^2}{8\sigma^2+c\epsilon} + \frac{1}{4}\right).
\end{equation*}
In particular, this implies that for any $\delta \in (0,1)$
\begin{equation*}
    \mathbb{P}\left(\left|\frac{1}{n}\sum_{i=1}^nX_i\right| \geq \frac{c \log(1/\delta)}{n}+\sqrt{\frac{8\sigma^2\log(1/\delta)}{n}}\right)\leq \delta e^{1/4} \;.
\end{equation*}
\end{lemma}
\section{In Expectation Convergence of Local Energy Clipping}\label{sec: In expectation convergence of Energy Clipping}
As mentioned in Remark \ref{rem: energy clipping}, it is common practice in variational Monte Carlo to clip the local energy during training. In this section, we show, using arguments similar to those in Section \ref{sec: Convergence in Expectation}, that clipping only the local energy is sufficient to guarantee convergence in expectation. However, unlike PS-Clip-VMC, this approach does not guarantee convergence with high probability.

\begin{setting}[VMC with Local Energy Clipping]\label{set: energy clipping optimization}
    Assume Setting \ref{set:Hamiltonian} and consider a parametrized model $\mathbb{R}^P\ni \theta \mapsto \psi_\theta\in H^2(\Omega)$. For $\theta\in \mathbb{R}^P$ define the \emph{loss} $L(\theta):= \frac{\langle \mathcal{H}\psi_\theta,\psi_\theta\rangle}{\|\psi_\theta\|_{L^2}^2}=\mathbb{E}[E_\theta]$.\\
    Let $L^\ast:=\inf_{\theta\in \mathbb{R}^P}L(\theta)$ and assume that $L^\ast >-\infty$. For $n\in \mathbb{N}$, $\alpha \in (0,\infty)$, $\theta\in \mathbb{R}^P$, and i.i.d samples $X_1, \ldots X_n \sim X_\theta$, set $E_i = E_{L,\theta}(X_i)$, $W_i = W_{\theta}(X_i)$ and let the \emph{clipped energy gradient estimator} be defined as
    \begin{equation}\label{eq:energy clipped_grad_est}
        \begin{aligned}
        G_n(\theta)&=\frac{1}{n-1}\sum_{k=1}^n \left(\gamma_{u,4-p,k}(E_k)-\frac{1}{n}\sum_{j=1}^n\gamma_{u,4-p,j}(E_j)\right)W_k\;,
        \end{aligned}
    \end{equation}
    where $\gamma_{\alpha,p,k}(\cdot)$ is the clipping function defined in (\ref{eq:clip_factor}).
    For $M,n\in \mathbb{N}$, $\theta_1\in \mathbb{R}^P$, $\alpha, \beta\in \mathbb{R}$, $\eta_1,\dots , \eta_{M-1}\in (0,\infty)$ define the \emph{Energy-Clip-VMC parameter update rule}
    $$
        \theta_{m+1}:=\theta_m - \eta_m G_n(\theta,\alpha),\quad m=1,\dots , M-1,
    $$
    where for each parameter update $\theta_m\mapsto \theta_{m+1}$, the clipped gradient estimator $G_n(\theta,\alpha)$ is evaluated on i.i.d. samples $E_1,\dots , E_n \sim E_{\theta_m}$ and $W_1,\dots , W_n \sim W_{\theta_m}$.
    As before, $\Delta_1:= L(\theta_1) - L^\ast>0$ denotes the initialization gap. 
\end{setting}
\begin{assumptions}\label{set: Energy Clipping assumptions}
Assume Settings \ref{set:Hamiltonian} and \ref{set: energy clipping optimization}, and that there are $p\in (1,2]$ and  $u, \sigma, C \in (0,\infty)$ such that for any $\theta \in \mathbb{R}^P$,
\begin{itemize}
\item[(i)] $\E \left[|E_\theta|^2\right]\leq \sigma^2$
\item[(ii)] $\E\left[\left| W_\theta\right|^2\right]\leq \sigma^2$
\item[(iii)] $\E\left[\left|E_\theta\right|^{2-p}\cdot \left|W_\theta\right|^{2}\right],\E\left[\left|E_\theta\right|^{3-p}\left|W_\theta\right|\right]\leq u^{4-p}$
\item[(iv)] $|\nabla_\theta L(\theta_1)-\nabla_\theta L(\theta_2)|\leq C|\theta_1-\theta_2|$ for any $\theta_1, \theta_2\in \mathbb{R}^P$.
\end{itemize}
\end{assumptions}
\begin{remark}
    Using the same arguments as in Section \ref{sec: Existence of Moments}, it is straightforward to show that the class of functions considered in Theorem \ref{thm:momentexistencegrad} also satisfy Assumptions \ref{set: Energy Clipping assumptions}(iii) for any $p\in (\frac{3}2{},2]$.
\end{remark}
\begin{remark}
While, as we will see in the following, local energy clipping guarantees convergence in expectation, it does not provide a high-probability error bound analogous to the one in Lemma \ref{lemm: high prob bound for estimator}. This can be seen from the fact that, for a fixed $\theta \in \R^P$, any bound of the form 
\begin{equation*}
    \mathcal{P}\left(\left|G_n(\theta)- \nabla_\theta L(\theta)\right|>\delta\right)\leq a e^{-b\delta^s}
\end{equation*}
in particular implies that $\E\left[|G_n(\theta)|^p\right]<\infty$ for any $p>0$. However, by the same arguments used in Section \ref{sec:nonexistencemoments}, $G_n(\theta)$ does not, in general, even possess a third moment.
\end{remark}
\begin{theorem}\label{thm: convergence in exp of energy clipping}
      Assume Settings \ref{set: energy clipping optimization} and Assumptions \ref{set: Energy Clipping assumptions}. Then, there is a constant $K_{u,\sigma}>0$ that only depends on $u$ and $\sigma$, such that any $\eta_m <1/C$
\begin{align*}
        \sum_{m=1}^M\frac{\eta_m\E(|\nabla_\theta L(\theta_m)|^2)}{\sum_{m=1}^M\eta_m} \leq \frac{2\Delta_1}{\sum_{m=1}^M\eta_m} + \frac{K_{u,\sigma}}{n^{\frac{2(2-p)}{4-p}}}\;.
    \end{align*}
    In particular, if $\eta_m = 1/(2C)$ is constant and $n = M^{\frac{4-p}{2(2-p)}}$ we obtain that 
    \begin{align*}
        \frac{1}{M}\sum_{m=1}^M\E(|\nabla_\theta L(\theta_m)|^2)\leq \frac{\Delta_1 C +K_{u,\sigma}}{M};.
    \end{align*}
\end{theorem}
The proof of Theorem \ref{thm: convergence in exp of energy clipping} follows the same steps as the one of Theorem \ref{thm: convergence in exp of vartiational monte carlo}. We begin with the following proposition.
\begin{proposition}\label{prop: generic in exp bound for energy clipping}
    Assume Settings \ref{set: energy clipping optimization} and Assumption \ref{set: Energy Clipping assumptions}(iii). Then,
    \begin{equation*}
        \E\left(\left|\frac{1}{n}\sum_{k=1}^n \gamma_{u,4-p,k}(E_k)W_k-\E(E_1W_1)\right|^2\right)\leq 8u^4n^{\frac{2(p-2)}{4-p}} \;.
    \end{equation*}
\end{proposition}
\begin{proof}
    For ease of notation, we set $\bare_k = E_k\gamma_{u,4-p,k}(E_k)$ and define $\mu = \E[E_1W_1]$ and $\mu_k = \E(\bare_k W_k)$. Then, we have that
    \begin{align*}
        |\mu-\mu_k| = |\E\left(\left(E_k-\bare_k\right)W_k\right)| \leq \E(|E_k \1_{(|E_k|>u\sqrt{k})}W_k|)\leq \E\left(|E_k|^{3-p}|W_k|\right)u^{p-2}k^{\frac{p-2}{4-p}}\leq u^{2}k^{\frac{p-2}{4-p}} \;.
    \end{align*}
Hence, we obtain that 
\begin{equation}\label{eq: in exp conv of energy clipping eq 1}
    \E\left(\left|\frac{1}{n}\sum_{k=1}^n (\mu_k-\mu)\right|^2\right)\leq
    \frac{1}{n^2}\E\left(\left|\sum_{k=1}^n u^{2}k^{\frac{p-2}{4-p}}\right|^2\right)\leq \frac{1}{n^2}\left(2u^2n^{\frac{2}{4-p}}\right)^2= 4u^4n^{\frac{2(p-2)}{4-p}}\;.
\end{equation}
Further, we have that
\begin{equation}\label{eq: in exp conv of energy clipping eq 2}
\begin{aligned}
    \E\left(\left|\frac{1}{n}\sum_{k=1}^n\bare_k-\mu_k\right|^2\right) &\leq \frac{1}{n^2}\sum_{k=1}^n\E\left(|\bare_k W_k|^2\right)\leq \frac{1}{n^2}\sum_{k=1}^n\E\left(|E_k|^{2-p}|W_k|^{2}\right)u^{p}k^{\frac{p}{4-p}}\\
    &\leq\frac{1}{n^2}4u^4n^{\frac{4}{4-p}}\leq 4u^4n^{\frac{2(p-2)}{4-p}}\;.
\end{aligned}
\end{equation}
Combining (\ref{eq: in exp conv of energy clipping eq 1}) and (\ref{eq: in exp conv of energy clipping eq 2}) we obtain that
\begin{equation}
    \E\left(\left|\frac{1}{n}\sum_{k=1}^n \bare_k W_k-\E(E_1W_1)\right|^2\right)\leq \E\left(\left|\frac{1}{n}\sum_{k=1}^n (\mu_k-\mu)\right|^2\right) +\E\left(\left|\frac{1}{n}\sum_{k=1}^n\bare_k-\mu_k\right|^2\right) \leq 8u^4n^{\frac{2(p-2)}{4-p}}\;.
\end{equation}
\end{proof}
Next, we provide an error bound for the double sum term in (\ref{eq:energy clipped_grad_est}).
\begin{proposition}\label{prop: mixed term bound for energy clipping}
      Assume Settings \ref{set: energy clipping optimization} and Assumptions \ref{set: Energy Clipping assumptions}(i-iii). Then,
    \begin{equation*}
        \E\left(\left|\frac{1}{n(n-1)}\sum_{k\neq j}^n \gamma_{u,4-p,k}(E_k)W_j-\E(E_1)\E(W_1)\right|^2\right)\leq \left(\frac{3}{n-1}+2\right)\frac{\sigma^4}{n}+{4u^4\sigma^2}{n^{\frac{2(p-2)}{4-p}}} \;.
    \end{equation*}
\end{proposition}
\begin{proof}
     We again set $\bare_k = \gamma_{u,4-p,k}(E_k)$ and define $\mu = \E[E_1W_1]$ and $\mu_k = \E(\bare_k W_k)$. Then, we have that
     \begin{equation}\label{eq: mixed term clip energy conv}
         \begin{aligned}
             \E\left(\left|\sum_{k\neq j} \bare_k W_j-\E(E_1)\E(W_1)\right|^2\right)&\leq \E\left(\left|\sum_{k\neq j} \bare_k W_j-\E(\bare_k)\E(W_j)\right|^2\right) \\
             &\quad+\E\left(\left|\sum_{k\neq j} (\E(\bare_k)-\E(E_k)) \E(W_j)\right|^2\right)\;.
         \end{aligned}
     \end{equation}
The first term on the right-hand side of (\ref{eq: mixed term clip energy conv}) can be estimated using the exact same steps as in (\ref{eq: first term rhs}) to obtain that
\begin{equation}\label{eq: mixed term 1}
    \E\left(\left|\sum_{k\neq j} \bare_k W_j-\E(\bare_k)\E(W_j)\right|^2\right) \leq (3n(n-1)+2n(n-1)(n-2))\sigma^4 \;.
\end{equation}
Further, for the second term in (\ref{eq: mixed term clip energy conv}) we have that
\begin{equation}\label{eq: mixed term 2}
    \E\left(\left|\sum_{k\neq j} (\E(\bare_k)-\E(E_k)) \E(W_j)\right|^2\right) \leq \E\left(\left|n\sigma\sum_{k=1}^n \E(\bare_k)-\E(E_k))\right|^2\right) \leq (n-1)^2\sigma^2 4u^4n^{\frac{4}{4-p}} \;,
\end{equation}
where the last inequality follows by (\ref{eq: in exp conv of energy clipping eq 2}). Combining (\ref{eq: mixed term 1}) and (\ref{eq: mixed term 2}) we obtain that
\begin{align*}
    \E\left(\left|\frac{1}{n(n-1)}\sum_{k\neq j} \bare_k W_j-\E(E_1)\E(W_1)\right|^2\right)&\leq \frac{(3+2(n-2))n(n-1)\sigma^4}{n^2(n-1)^2} +  \frac{\sigma^2 4u^4n^{\frac{4}{4-p}}}{n^2}\\
    &\leq \left(\frac{3}{n-1}+2\right)\frac{\sigma^4}{n}+{4u^4\sigma^2}{n^{\frac{2(p-2)}{4-p}}}
\end{align*}
\end{proof}
Combining Proposition \ref{prop: generic in exp bound for energy clipping} and Proposition \ref{prop: mixed term bound for energy clipping} it is now straightforward to prove the following Lemma.
\begin{lemma}\label{lem: error bound of energy clipping}
      Assume Settings \ref{set: energy clipping optimization} and Assumptions \ref{set: Energy Clipping assumptions}(i)-(iii). Then,
    \begin{equation*}
        \E\left(\left|\frac{1}{n-1}\sum_{k=1}^n \left(\gamma_{u,4-p,k}(E_k)-\frac{1}{n}\sum_{j=1}^n\gamma_{u,4-p,j}(E_j)\right)W_k-\E\big((E_1-\E(E_1))W_1\big)\right|^2\right) = O\left(n^{\frac{2(p-2)}{4-p}}\right)  \;.
    \end{equation*}
\end{lemma}
\begin{proof}
    Setting $\bare_k = \gamma_{u,4-p,k}(E_k)$, we obtain by Propositions \ref{prop: generic in exp bound for energy clipping} and \ref{prop: mixed term bound for energy clipping} that
    \begin{equation}
    \begin{aligned}
        \E\left(\left|\frac{1}{n-1}\sum_{k=1}^n \left(\bare_k-\frac{1}{n}\sum_{j=1}^n\bare_j\right)W_k-\E\big((E_1-\E(E_1))W_1\big)\right|^2\right)&\leq 2\E\left(\left|\frac{1}{n}\sum_{k=1}^n \bare_k W_k-\E(E_1W_1)\right|^2\right) \\  
        & +2\E\left(\left|\sum_{k\neq j} \bare_k W_j-\E(E_1)\E(W_1)\right|^2\right)\\
        &= O\left(n^{\frac{2(p-2)}{4-p}}\right)\;.
    \end{aligned}
    \end{equation}
\end{proof}
The proof of Theorem \ref{thm: convergence in exp of energy clipping} now follows by performing the same steps as in the proof of Theorem \ref{thm: convergence in exp of vartiational monte carlo} with the error bound given by Lemma \ref{lem: error bound of energy clipping}.
\end{document}